\documentclass[sigconf]{acmart}
\AtBeginDocument{
  }

\copyrightyear{2026}
\acmYear{2026}
\setcopyright{cc}
\setcctype{by}
\acmConference[KDD '26]{Proceedings of the 32nd ACM SIGKDD Conference on Knowledge Discovery and Data Mining V.2}{August 09--13, 2026}{Jeju Island, Republic of Korea}
\acmBooktitle{Proceedings of the 32nd ACM SIGKDD Conference on Knowledge Discovery and Data Mining V.2 (KDD '26), August 09--13, 2026, Jeju Island, Republic of Korea}
\acmDOI{10.1145/3770855.3818149}
\acmISBN{979-8-4007-2259-2/2026/08}

\usepackage{algorithm}
\usepackage{algorithmic}
\usepackage{multirow}
\usepackage{booktabs}
\usepackage{bbding}
\usepackage{pifont}
\usepackage{utfsym}
\usepackage{amsmath}
\usepackage{makecell}
\usepackage{array}
\usepackage{amsmath}
\usepackage[inline]{enumitem}
\usepackage{graphicx}
\usepackage{adjustbox}
\usepackage{csquotes}
\usepackage{newfloat}
\usepackage{listings}
\usepackage{bm}
\usepackage{graphicx}
\usepackage{subcaption}
\usepackage{makecell}
\usepackage{threeparttable}
\usepackage{hyperref}
\usepackage{colortbl}
\usepackage{amsthm}
\newtheorem*{remark}{Remark}
\usepackage{diagbox}
\usepackage{caption}
\usepackage{subcaption}
\usepackage[linesnumbered,ruled,vlined,algo2e]{algorithm2e}
\usepackage{tabularx} 
\usepackage{ragged2e} 
\usepackage{enumitem}

\DeclareMathOperator{\rank}{rank}
\DeclareMathOperator{\col}{col}
\DeclareMathOperator{\diag}{diag}

\newtheorem{theorem}{Theorem}
\newtheorem{lemma}{Lemma}
\newtheorem{proposition}{Proposition}
\newtheorem{corollary}{Corollary}
\newtheorem{assumption}{Assumption}
\theoremstyle{definition}
\newtheorem{definition}{Definition}

\newcommand{\eqcomma}{\,,}
\newcommand{\eqperiod}{\,.}

\begin{document}

\title{DiffoR: A Unified Continuous Generative Framework for Universal Ordinal Regression}

\author{Hongxu Ma}
\authornote{Both authors contributed equally to this research.}
\authornote{Work done during the internship at Kuaishou Technology.}
\affiliation{
  \institution{Fudan University}
  \city{Shanghai}
  \country{China}}
\email{hxma24@m.fudan.edu.cn}

\author{Lin Wang}
\authornotemark[1]
\affiliation{
  \institution{Kuaishou Technology}
  \city{Beijing}
  \country{China}}
\email{wanglin36@gmail.com}

\author{Chenghou Jin}
\authornotemark[1]
\affiliation{
  \institution{Fudan University}
  \city{Shanghai}
  \country{China}}
\email{jinch24@m.fudan.edu.cn}

\author{Han Zhou}
\affiliation{
  \institution{Shanghai University of Finance and Economics}
  \city{Shanghai}
  \country{China}}
\email{zhouhan@stu.sufe.edu.cn}

\author{Jie Zhang}
\affiliation{
  \institution{Kuaishou Technology}
  \city{Beijing}
  \country{China}}
\email{zhangjie39@kuaishou.com}

\author{Xiaoyu Yang}
\affiliation{
  \institution{Kuaishou Technology}
  \city{Beijing}
  \country{China}}
\email{yangxiaoyu@kuaishou.com@gmail.com}

\author{Chunjie Chen}
\affiliation{
  \institution{Kuaishou Technology}
  \city{Beijing}
  \country{China}}
\email{chencj517@gmail.com}

\author{Jihong Guan}
\affiliation{
  \institution{Tongji University}
  \city{Shanghai}
  \country{China}}
\email{jhguan@tongji.edu.cn}

\author{Shuigeng Zhou}
\authornote{Corresponding author.}
\affiliation{
  \institution{Fudan University}
  \city{Shanghai}
  \country{China}}
\email{sgzhou@fudan.edu.cn}
\renewcommand{\shortauthors}{Hongxu Ma et al.}

\begin{abstract}
Ordinal Regression (OR) aims to predict target values with inherent order,
underpinning critical applications across diverse domains, from recommender systems to computer vision.
Though having evolved from naive regression to discretization-based classification and generation, existing paradigms remain fundamentally constrained by quantization artifacts and the lack of global ordinal topological perception.
These methods typically enforce rigid boundary delineations, failing to capture the non-stationary semantic transitions inherent to ordinal data.

In this paper, we propose a novel paradigm where OR is formulated as a \textbf{Continuous Generative Ordinal Regression} task. Under the novel paradigm, we introduce \textbf{DiffoR}, a unified framework that leverages diffusion models to recover continuous ordinal values via iterative denoising, thereby enabling the dynamic learning of soft semantic transitions.
To explicitly preserve ordinal topology, we devise a Dual-Decoupling Strategy: Spatially, \textit{Multi-scale Increment Aggregation} decomposes targets into hierarchical continuous increments;  Temporally, \textit{Dynamic Denoising Perception} synchronizes denoising steps with feature frequencies, 
ensuring robust coarse-to-fine refinement.
Theoretically, we show that the proposed method can significantly enhance both representation capability and mechanistic interpretability. Extensive experiments on 12 benchmarks across four domains validate DiffoR's consistent superiority over state-of-the-art methods, establishing a new standard
that demonstrates strong potential as a general-purpose solution for universal ordinal regression.

\end{abstract}


\maketitle

\section{Introduction}
\label{sec:intro}
\begin{figure}[t]
    \centering
    \includegraphics[scale=0.5]{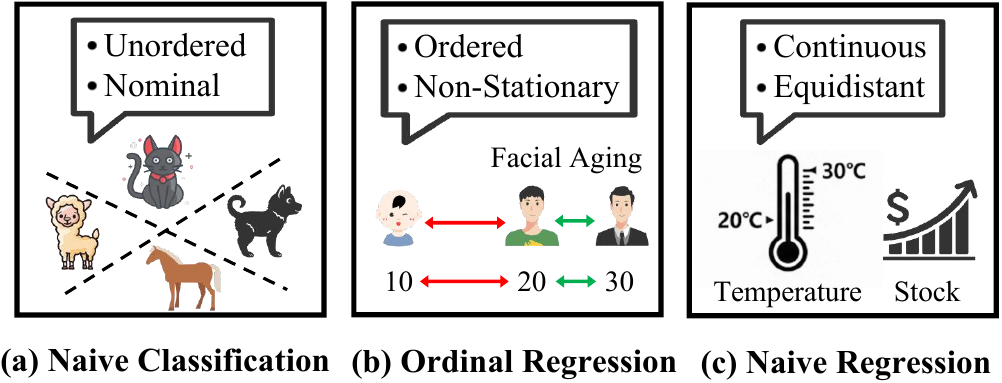}
    \caption{Conceptual comparison of (a) Naive Classification, (b) Ordinal Regression, and (c) Naive Regression. Ordinal Regression uniquely addresses data that is ordered but non-stationary, distinct from the nominal nature of classification and the equidistant metric of regression.}
    \label{fig:intro}
\end{figure}

\begin{table*}[t]
    \centering
    \caption{Evolution of Ordinal Regression Paradigms. `L' and `A' denote `Limitation' and `Advantage', respectively.}
    \label{tab:paradigm_comparison}
    \resizebox{0.95\linewidth}{!}{
    \begin{tabular}{l l l l}
        \toprule
        \textbf{Paradigm} & \textbf{Core Mechanism} & \textbf{Target Representation} & \textbf{Intrinsic Limitation / Advantage} \\
        \midrule
        \textit{Naive Regression} & Point-wise Mapping & Deterministic Point & \textbf{L:} Fails to capture complex ordinal distributions. \\
        \addlinespace[1pt]
        \textit{Space Discretization} & Space Quantization & Hard Categories & \textbf{L:} Disregards inherent ordinal relationships. \\
        \addlinespace[1pt]
        \textit{Rank-based} & Binary Decomposition & Independent Bits & \textbf{L:} Struggles with interval dependencies. \\
        \addlinespace[1pt]
        \textit{Discrete Generation} & Token Autoregression & Quantized Tokens & \textbf{L:} Heuristic vocabulary and quantization error. \\
        \midrule
        \rowcolor{gray!12} 
        \textbf{Continuous Generation} & \textbf{Manifold Diffusion} & \textbf{Continuous Distribution} & \textbf{A:} Captures soft transitions \& global topology. \\
        \bottomrule
    \end{tabular}
    }
\end{table*}
Predicting target values with inherent ordering, known as Ordinal Regression (OR), is instrumental across domains such as computer vision (e.g., facial age estimation~\cite{OR-CNN, Unimodal}), medical diagnosis (e.g., disease staging~\cite{wang2023ord2seq}) and recommendation systems (e.g. preference ranking~\cite{cread,ma2024generative,jin2026invariant,drachen2018or}). 
As illustrated in Fig.~\ref{fig:intro}, unlike classification (which ignores order) or metric regression (which assumes equidistant intervals), OR must model non-stationary semantic boundaries (e.g., the perceptual gap between ages 20→21 is smaller than 60→61 in facial aging). This inherent non-uniformity defies rigid discretization~\cite{wang2025survey}, a limitation we theoretically verify in Appendix~\ref{app:meancollapse} and \ref{app:ordinal}.

From the perspective of paradigm evolution (summarized in Tab.~\ref{tab:paradigm_comparison}), early approaches largely treat OR as a \textit{Naive Regression} task~\cite{rothe2015dex}, attempting to fit targets via point-wise mapping but often failing to capture complex ordinal distributions. Subsequently, \textit{Continuous Space Discretization} (CSD)~\cite{wang2025survey,diaz2019soft} becomes the dominant paradigm, simplifying regression into classification via space quantization. 
Its derivative, \textit{Rank-based} methods~\cite{OR-CNN,cread}, encode ordinality via binary subtasks yet struggle with inherent interval independence and rigid predictions.
Recently, analogous to Large Language Models (LLMs), \textit{Discrete Generative} methods~\cite{wang2023ord2seq, ma2026gor} reformulate OR as sequence generation, but remain hampered by heuristic vocabulary design and inevitable quantization error.
Crucially, as we theoretically show in Appendix.~\ref{app:ordinal} and \ref{app:ar_bottleneck}, these discretization-based formulations introduce deterministic boundaries into an inherently continuous ordinal space. As a result, predictions are often fragmented locally, insensitive to ordinal distances, and blind to the global topology of the ordinal space, especially near semantic transition regions.
While a number of works attempt to enhance encoders~\cite{wang2017zoom,yang2023diffmic,zhang2025multi,ma2025fine,ma2025ms,feng2026muvr} or tackle boundary ambiguity~\cite{li2021learning,shin2022moving}, these optimizations mostly address domain-specific symptoms rather than the intrinsic limitations of the underlying discrete paradigms, remaining inherently palliative and lacking cross-task generalization.

This situation poses a fundamental question: \textit{Is there a paradigm that can circumvent the intrinsic limitations of discretization while natively accommodating the strict ordinal dependencies and non-stationary boundaries of OR?}

A natural alternative is to directly regress continuous targets. However, naive regression collapses the ordinal structure into a single scalar estimate, failing to explicitly represent hierarchical ordinal semantics or capture uncertainty across different scales. In practice, coarse ordinal levels (e.g., ``which stage'') and fine-grained variations (e.g., ``where within a stage'') are governed by distinct semantic cues, which are difficult to disentangle within a single scalar prediction.

In response to this question, this paper turns to Continuous Generative Models~\cite{vae,ddpm,flowmatching} and posits that their exceptional capability in characterizing complex probability density manifolds enables the adaptive learning of latent target distributions across diverse domains, offering a principled solution to model the continuous ordinal space.
Unlike traditional discretization that enforces deterministic boundary delineation, this continuous generative capability enables the dynamic learning of soft semantic transitions, thereby perfectly capturing the non-stationary gradations between adjacent ordinal values.
To this end, we formally propose a novel paradigm --- \textbf{Continuous Generative Ordinal Regression}, and introduce \textbf{DiffoR}, a unified diffusion-based framework that models ordinal targets as a structured multivariate continuous distribution via joint generation rather than discrete classification.


Different from prior paradigms, we reformulate OR as a conditional continuous value recovery process, progressively estimating the target via iterative denoising. To inject ordinal priors and explicitly preserve the ordinal topology within the generation process, we devise \textit{Multi-scale Increment Aggregation}, which models the prediction as a \emph{multivariate continuous distribution}. This mechanism decouples the target into hierarchical continuous increments, spanning from global coarse-grained estimations to local fine-grained refinements. 
Instead of independent predictions, our model jointly generates all components from a shared latent representation, ensuring cross-scale ordinal dependency and global consistency.
By dedicating distinct attention heads in the Transformer to encode features at specific scales, we leverage parallel attention subspaces to capture multi-granular ordinal dependencies, enabling efficient joint decoding.


Furthermore, mirroring the coarse-to-fine cognitive perception where macroscopic judgments (e.g., ``elderly'') rely on low-frequency structures and microscopic estimations (e.g., ``65 years'') demand high-frequency details, we introduce a \textit{Dynamic Denoising Perception} strategy. 
Recognizing that distinct ordinal increments align with varying feature frequency spectra, we discard the unified timestep convention of standard diffusion models. By injecting scale-adaptive stochastic perturbations to decouple the denoising process, we enable the model to learn robust representations for coarse scales under high noise levels, while reserving low-noise steps for fine-grained calibration. This dynamic synchronization effectively aligns the denoising trajectory with the hierarchical semantics of the target.

Overall, the contributions of this paper are as follows:
\begin{itemize}[leftmargin=20pt]
    \item We provide a rigorous theoretical analysis that exposes the limitations inherent to existing paradigms and propose a novel paradigm that formalizes ordinal regression as a Continuous Generative task to bypass quantization limits.
    \item We develop \textbf{DiffoR}, a unified continuous regression framework tailored for OR that adapts diffusion generation to capture diverse latent ordinal spaces. By leveraging soft semantic transitions, it accurately captures the non-stationary gradations between adjacent values,
    \item We design a Dual-Decoupling Strategy, comprising \textit{Multi-scale Increment Aggregation} to construct a structured ordinal space and \textit{Dynamic Denoising Perception} to align semantic granularity with feature frequency. Theoretical analysis confirms that this design significantly enhances the model's representation capability and interpretability.
    \item Extensive experiments across 12 ordinal regression benchmarks spanning four domains demonstrate DiffoR's strong generalization and consistent superiority over the SOTAs.
\end{itemize}

\section{Related Work}
\subsection{Ordinal Regression (OR)}
OR addresses prediction tasks with ordered targets, widely applied in diverse domains like facial age estimation~\citep{niu2016ordinal,chen2017using}, image aesthetic/quality assessment~\citep{he2022rethinking,he2023thinking}, watch-time prediction~\cite {cread,tpm,ma2024generative,jin2026invariant}, life-time value prediction~\citep{wang2019deep,li2022billion,weng2024optdist}.
The landscape of OR has witnessed a shift from continuous mapping to discrete modeling.
Initial \textit{Naive Regression} methods~\citep{rothe2015dex} utilize standard regression losses but struggle with non-uniform ordinal intervals.
Continuous Space Discretization (CSD)~\cite{wang2025survey,diaz2019soft} and Rank-based Methods~\citep{OR-CNN,cread,tpm} transform regression into classification or binary subtasks. Despite their popularity, they still suffer from inherent quantization artifacts and rigid boundary delineations that fragment the global ordinal space.
Discrete Generative Models~\citep{wang2023ord2seq,ma2026gor} emulate LLMs to generate ordinal sequences, yet remain bound by discrete vocabularies and heuristic tokenization.
While CLIP-based methods~\citep{li2022ordinalclip,yu2024clip,wang2023learning,du2024teach} enhance semantic understanding via vision-language alignment, but they do not address the fundamental limitations of discretization.
In this paper, our DiffoR, the first Continuous Generative framework for OR, models the target distribution as a continuous manifold via diffusion, can effectively capture the soft and non-stationary transitions between ordinal values.

\subsection{Generative Regression Modeling}
Generative modeling has traditionally been bifurcated into discrete and continuous paradigms.
Discrete Generative Models, exemplified by Large Language Models (LLMs)~\citep{gpt3, llama,wang2026creativebench,xu-etal-2025-alignment,xu2025reducing}, excel in capturing sequential dependencies via autoregressive token prediction. While recently adapted for Ordinal Regression~\citep{wang2023ord2seq,ma2026gor}, these methods are inherently constrained by heuristic vocabularies and quantization errors, limiting their precision estimation.
Conversely, Continuous Generative Models, such as Diffusion~\citep{ddpm} and Flow Matching~\citep{flowmatching}, have surged in prominence due to their exceptional capability in modeling complex, high-dimensional probability density manifolds.
Leveraging this capability, recent works have successfully extended diffusion models to standard classification~\citep{card,yang2023diffmic,uliana2025diffusion,chen2023robust} and regression~\citep{zhao2019variational}, demonstrating superior robustness and uncertainty quantification.
However, the application of continuous generation to Ordinal Regression remains unexplored, primarily due to the difficulty of capturing strict ordinal relationships and non-stationary semantic boundaries—issues, which is addressed in this work.

\section{Preliminaries}

\subsection{Conditional Diffusion Models}
Diffusion Models ~\cite{ho2020denoising} learn complex data distributions by simulating a gradual noise injection and removal process. In this work, we focus on \emph{conditional latent diffusion models}, where the generation process is guided by auxiliary conditions $\mathbf{x}$ (i.e., the input features).

\subsubsection{Forward Diffusion Process.}
Given an initial latent representation $\mathbf{z}_0$, Gaussian noise is progressively injected through a Markov chain controlled by a predefined variance schedule $\beta_t \in (0,1)$:
\begin{equation}
q(\mathbf{z}_t \mid \mathbf{z}_{t-1})
= \mathcal{N}\!\left(\mathbf{z}_t;\sqrt{1-\beta_t}\mathbf{z}_{t-1},\beta_t \mathbf{I}\right).
\end{equation}
Let $\alpha_t = 1 - \beta_t$ and $\bar{\alpha}_t = \prod_{i=1}^t \alpha_i$. Using the reparameterization trick, the latent variable at time step $t$ can be expressed as:
\begin{equation}
\mathbf{z}_t
= \sqrt{\bar{\alpha}_t}\mathbf{z}_0
+ \sqrt{1 - \bar{\alpha}_t}\,\boldsymbol{\epsilon},
\quad \boldsymbol{\epsilon} \sim \mathcal{N}(0, \mathbf{I}).
\end{equation}
When the total number of diffusion steps $T$ becomes sufficiently large, $\mathbf{z}_T$ approaches an isotropic Gaussian distribution.

\subsubsection{Reverse Denoising Process.}
The reverse process aims to learn a parameterized conditional distribution
\begin{equation}
p_\theta(\mathbf{z}_{t-1} \mid \mathbf{z}_t, \mathbf{x}),
\end{equation}
which iteratively removes noise from $\mathbf{z}_t$ conditioned on $\mathbf{x}$. At each time step, the model predicts either the injected noise or the original latent representation. When predicting noise, the reverse update takes the form:
\begin{equation}
\mathbf{z}_{t-1}
= \frac{1}{\sqrt{\alpha_t}}
\left(
\mathbf{z}_t
- \frac{1-\alpha_t}{\sqrt{1-\bar{\alpha}_t}}
\epsilon_\theta(\mathbf{z}_t, \mathbf{x}, t)
\right)
+ \sigma_t \mathbf{v},
\end{equation}
where $\mathbf{v} \sim \mathcal{N}(0, \mathbf{I})$ and $\sigma_t$ controls  sampling process stochasticity.

\section{Methodology}


This section instantiates the paradigm described in the Introduction: we treat ordinal regression as \textbf{continuous generative ordinal regression} over an ordered latent space. DiffoR maps $\mathbf{x}$ to conditional features $\mathbf{c}=E_\psi(\mathbf{x})$, runs a conditional diffusion denoiser $\epsilon_\theta(\mathbf{z}_t,\mathbf{c},t)$ to obtain denoised latents $\{\hat{\mathbf{z}}_0(t)\}$, and decodes an \emph{order-aware} increment vector $\hat{\mathbf{B}}$ whose sum yields the final prediction $\hat{y}$.
Our key structural novelty is \textbf{spatio-temporal decoupling}, implemented by two modules: \textbf{MIA} (Multi-scale Increment Aggregation) \emph{spatially} decomposes the target into $S$ ordered continuous increments and assigns each increment to a dedicated head; \textbf{DDP} (Dynamic Denoising Perception) \emph{temporally} aligns each head to a dedicated denoising timestep $t_k$, so coarse-to-fine semantics are decoded from noise levels matched to feature frequency. The rest of this section introduces the increment target space, the conditional diffusion backbone, the dual-decoupled decoding rule, and the training/inference procedures.

\paragraph{Overall architecture.}
DiffoR consists of: (i) an encoder $E_\psi$ producing conditional features $\mathbf{c}=E_\psi(\mathbf{x})$; (ii) a diffusion denoiser (backbone) $\epsilon_\theta(\mathbf{z}_t,\mathbf{c},t)$; (iii) $S$ scale-specific decoding heads $\{\mathcal D_k\}_{k=1}^S$; and (iv) an order-preserving aggregation rule.

As illustrated in Fig.~\ref{fig:framework}, the overall pipeline is: (i) compute $\mathbf{c}=E_\psi(\mathbf{x})$; (ii) perform conditional diffusion to obtain denoised latent estimates $\{\hat{\mathbf{z}}_0(t)\}$; (iii) decode multi-scale increments $\hat{\mathbf{B}}$ from aligned timesteps; (iv) aggregate $\hat{\mathbf{B}}$ to produce the final ordinal prediction $\hat{y}$.
Crucially, DiffoR follows a \textbf{Dual-Decoupling} design (matching the Contributions): \textbf{spatially}, \emph{Multi-scale Increment Aggregation (MIA)} decomposes targets into hierarchical continuous increments and assigns each increment to a dedicated head; \textbf{temporally}, \emph{Dynamic Denoising Perception (DDP)} \emph{synchronizes denoising steps with semantic granularity / feature frequency}, so coarse semantics are perceived under higher noise while fine semantics are calibrated under lower noise. This design ensures robust coarse-to-fine refinement and yields provable guarantees (Sec.~\ref{sec:theory-main}).



\subsection{Problem Formulation for Multi-scale Increment Representation}
Given an normalized ordinal target $y \in [0,1]$, we divide the target range into $S$ ordered intervals with uniform width $\Delta = 1/S$. With this partition, we represent each target using an ordered $S$-dimensional continuous vector
\begin{equation}
\mathbf{B} = [b_1, b_2, \dots, b_S]^\top,
\end{equation}
where each component corresponds to the \emph{usable completion} of $y$ within the $k$-th ordinal interval. We use the explicit construction
\begin{equation}
b_k
\;=\;
\min\Big\{\max\big(y-(k-1)\Delta,\,0\big),\,\Delta\Big\},
\qquad k=1,\dots,S,
\label{eq:bk-definition}
\end{equation}
which guarantees $b_k\in[0,\Delta]$ and yields an exact additive reconstruction:
\begin{equation}
y = \sum_{k=1}^{S} b_k.
\end{equation}
This representation induces a hierarchical structure: lower-index components encode coarse ordinal completion, while higher-index components remain sensitive near semantic transition regions.

\subsection{Conditional Diffusion Backbone with Encoder and Denoiser}
\label{sec:diffusion-backbone}
To model the conditional distribution over $\mathbf{B}$, we adopt a conditional diffusion framework. Given input $\mathbf{x}$, we first encode it as $\mathbf{c}=E_\psi(\mathbf{x})$. We assume the structured ordinal information is embedded in a latent variable $\mathbf{z}_0$ on a continuous ordinal manifold, and we learn a conditional denoiser to recover $\mathbf{z}_0$ from its noised versions. Since $\mathbf{c}$ is a deterministic function of $\mathbf{x}$, conditioning $\epsilon_\theta(\cdot)$ on $\mathbf{c}$ is equivalent to conditioning on $\mathbf{x}$ as in standard conditional diffusion notation.
Using the standard reparameterization, the noised latent at time $t$ can be written as
\begin{equation}
\mathbf{z}_t
=
\sqrt{\bar{\alpha}_t}\mathbf{z}_0
+\sqrt{1-\bar{\alpha}_t}\,\boldsymbol{\epsilon},
\qquad
\boldsymbol{\epsilon}\sim\mathcal{N}(0,\mathbf{I}).
\end{equation}
We parameterize the reverse model by \emph{noise prediction} $\epsilon_\theta(\mathbf{z}_t,\mathbf c,t)$ and form the standard estimator
\begin{equation}
\hat{\mathbf{z}}_0(t)
=
\frac{1}{\sqrt{\bar{\alpha}_t}}
\left(
\mathbf{z}_t-\sqrt{1-\bar{\alpha}_t}\,\epsilon_\theta(\mathbf{z}_t,\mathbf c,t)
\right).
\end{equation}

\begin{figure*}
    \centering
    \includegraphics[scale=0.75]{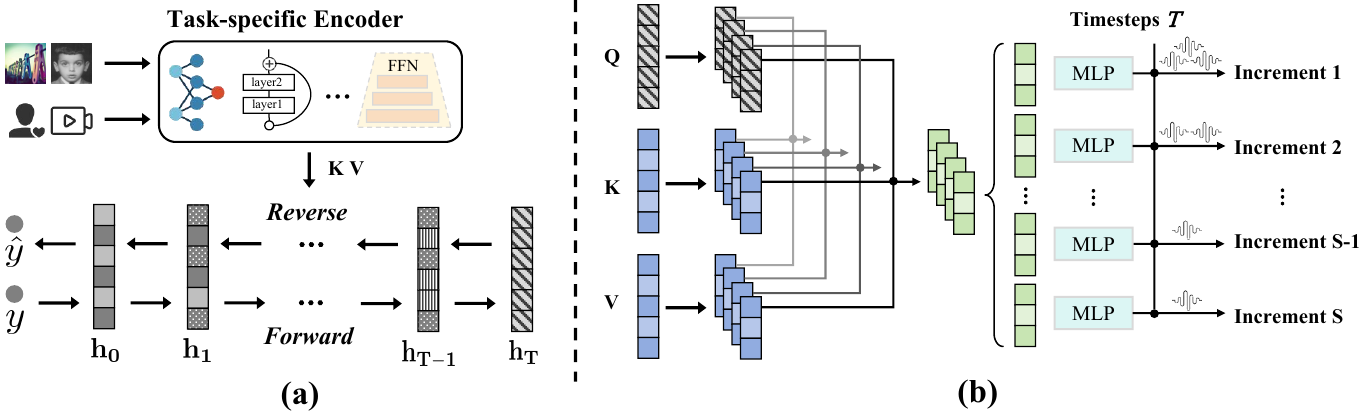}
    \caption{Overview of the proposed Generative Ordinal Regression framework DiffoR. (a) The overall pipeline operates as a conditional generative process. It consists of a forward diffusion process that perturbs the label $y$ into a latent space $\mathbf{h}_T$, and a reverse denoising process that recovers $\hat{y}$ using features from a task-specific encoder. (b) The detailed architecture of the Denoising Module with Multi-scale Increment Aggregation and Dynamic Denoising Perception.}
    \label{fig:framework}
\end{figure*}

\subsection{Dual-Decoupling Strategy}
\subsubsection{Spatial decoupling: Multi-scale Increment Aggregation (MIA)}
MIA implements the \textbf{spatial decoupling} described in the Introduction: instead of predicting $y$ with a single head, we decode the increment vector $\hat{\mathbf B}$ using $S$ parallel \emph{scale-specific heads} $\{\mathcal{D}_1,\dots,\mathcal{D}_S\}$, each specializing in one ordinal increment $b_k$. Each head predicts a bounded scalar increment in $[0,\Delta]$ by
\begin{equation}
\hat b_k \;=\; \Delta\cdot \sigma\!\left(\mathrm{MLP}_k(\mathbf h_k)\right),\qquad k=1,\dots,S,
\label{eq:mia-decoding}
\end{equation}
where $\sigma(\cdot)$ denotes the sigmoid function. Importantly, the heads are not forced to share the same denoising state; instead, DDP below specifies the scale-specific input $\mathbf h_k$ as an aligned denoising state.
This design is motivated by the non-stationary nature of ordinal semantics: coarse-level decisions (e.g., ``which interval'') and fine-level calibration (e.g., ``where within the interval'') are governed by different cues and should not be entangled into a single scalar readout. By representing $y$ as an \emph{ordered sum} of bounded increments (Eq.~\ref{eq:bk-definition}), each head becomes interpretable and ablatable at the interval level, the bounded range $\hat b_k\in[0,\Delta]$ stabilizes decoding and facilitates the order-preserving aggregation (Eq.~\ref{eq:usable-increments}), and the multi-head structure encourages specialization across semantic scales under a fixed overall model size. All heads share the same conditional diffusion backbone and only differ in lightweight head parameters (the $\mathrm{MLP}_k$'s), so the additional cost is $O(S)$ small MLPs on top of one shared denoiser.

\subsubsection{Temporal decoupling: Dynamic Denoising Perception (DDP)}
\label{sec:ddp}
DDP implements the \textbf{temporal decoupling} described in the Introduction: it \emph{synchronizes denoising steps with semantic granularity / feature frequency}. Coarse semantics are dominated by low-frequency structures and are robust under higher noise, while fine semantics require high-frequency details and are better calibrated under lower noise.
Accordingly, DDP aligns semantic granularity to denoising noise levels by assigning each scale head $k$ a dedicated timestep $t_k$:
\begin{equation}
t_1 > t_2 > \cdots > t_S,\qquad t_k \in \{1,\dots,T\}.
\label{eq:tk-schedule}
\end{equation}
Smaller $k$ corresponds to coarser increments and uses larger noise (larger $t_k$), while larger $k$ corresponds to finer increments and uses smaller noise (smaller $t_k$). The schedule $\{t_k\}$ is a simple hyper-parameter (e.g., linear or log-spaced over $[1,T]$).
In our implementation we use a monotone spacing over $[1,T]$, e.g., a simple linear map,
\begin{equation}
t_k \;=\; 1+\left\lfloor \frac{S-k}{S-1}(T-1)\right\rfloor,\qquad k=1,\dots,S,
\label{eq:tk-linear}
\end{equation}
and we also consider log-spaced schedules in ablations. This explicit mapping makes DDP fully reproducible.

Each head consumes its aligned denoised estimate:
\begin{equation}
\mathbf h_k \;=\; \hat{\mathbf z}_0(t_k),
\qquad
\hat b_k \;=\; \Delta\cdot \sigma\!\left(\mathrm{MLP}_k(\mathbf h_k)\right),
\qquad k=1,\dots,S.
\label{eq:ddp-decoding}
\end{equation}
This realizes temporal decoupling: different semantic scales are decoded from different denoising states rather than from a single shared timestep representation.
During inference, we cache $\hat{\mathbf z}_0(t_k)$ when the denoising chain reaches $t_k$ and decode $\hat b_k$ by Eq.~\ref{eq:ddp-decoding} (one pass, negligible overhead).
Besides the standard diffusion loss sampled at a random timestep $t$, we explicitly supervise the scale heads at their aligned timesteps $\{t_k\}$ via $\mathcal L_{\mathrm{mia}}$ (Eq.~\ref{eq:ddp-decoding}). To emphasize DDP during diffusion training, we use a mixture timestep sampler
\begin{equation}
p(t) \;=\; (1-\beta)\cdot \mathrm{Unif}\{1,\dots,T\} \;+\; \beta\cdot \frac{1}{S}\sum_{k=1}^S \delta_{t_k}(t),
\label{eq:t-mix}
\end{equation}
where $\mathrm{Unif}\{1,\dots,T\}$ denotes the uniform distribution over all timesteps and $\delta_{t_k}(t)$ denotes a point mass at $t=t_k$ (i.e., sampling $t_k$ with probability $1$). Equivalently, Eq.~\ref{eq:t-mix} can be implemented by: with probability $1-\beta$, sample $t$ uniformly from $\{1,\dots,T\}$; with probability $\beta$, sample $t$ uniformly from the aligned set $\{t_1,\dots,t_S\}$.
The purpose of this mixture is twofold: (i) the uniform term maintains global coverage of the denoising trajectory, preserving the standard diffusion training behavior; (ii) the aligned-set term increases the training frequency of the DDP-critical timesteps, making the denoised states $\{\hat{\mathbf z}_0(t_k)\}$ more reliable for coarse-to-fine decoding. The hyper-parameter $\beta\in[0,1]$ controls the trade-off between global trajectory learning ($\beta\!\downarrow$) and aligned-step emphasis for DDP ($\beta\!\uparrow$).
The complete training and inference procedure is summarized in Algorithm~\ref{alg:diffior}.

A standard conditional diffusion regressor typically decodes the prediction from a \emph{single} denoising state (often the final sample), i.e.,
\begin{equation}
\text{(single-time, single-head)}\qquad
\hat y_{\mathrm{std}} \;=\; g\!\left(\hat{\mathbf z}_0(t_\star),\,\mathbf c\right),
\label{eq:std-diffusion-decoding}
\end{equation}
for some fixed $t_\star$ (or using the final denoised output). In contrast, DiffoR jointly uses \emph{multiple} aligned timesteps $\{t_k\}$ and \emph{multiple} scale heads to generate a structured increment vector and then aggregates it as follows:
\begin{equation}
\text{(DiffoR)}\quad
\hat b_k=\mathcal D_k\!\left(\hat{\mathbf z}_0(t_k)\right),\ \ 
\hat{\mathbf B}=[\hat b_1,\dots,\hat b_S]^\top,\ \ 
\hat y=\mathrm{Agg}(\hat{\mathbf B}),
\label{eq:our-decoding-summary}
\end{equation}
matching the spatio-temporal decoupling advocated in Sec.~\ref{sec:intro}.

\begin{algorithm}[t]
\caption{DiffoR: training and inference with scale--time aligned decoding}
\label{alg:diffior}
\begin{algorithmic}[1]
\STATE \textbf{Inputs:} sample $(\mathbf x,y)$, intervals $S$ (width $\Delta=1/S$), aligned timesteps $\{t_k\}_{k=1}^S$, weights $\lambda$
\STATE Compute conditional features $\mathbf c \leftarrow E_\psi(\mathbf x)$
\STATE Construct targets $b_k \leftarrow \min\{\max(y-(k-1)\Delta,0),\Delta\}$ for $k=1,\dots,S$ \hfill (Eq.~\ref{eq:bk-definition})
\STATE Sample $\boldsymbol\epsilon\sim\mathcal N(0,\mathbf I)$ and sample a timestep $t$ for diffusion training
\STATE Form $\mathbf z_t \leftarrow \sqrt{\bar\alpha_t}\mathbf z_0+\sqrt{1-\bar\alpha_t}\boldsymbol\epsilon$ \hfill (forward diffusion)
\STATE Predict noise $\hat{\boldsymbol\epsilon}\leftarrow \epsilon_\theta(\mathbf z_t,\mathbf c,t)$ and compute $\mathcal L_{\mathrm{diff}} \leftarrow \|\boldsymbol\epsilon-\hat{\boldsymbol\epsilon}\|_2^2$
\FOR{$k=1$ \TO $S$}
\STATE Form $\mathbf z_{t_k}\leftarrow \sqrt{\bar\alpha_{t_k}}\mathbf z_0+\sqrt{1-\bar\alpha_{t_k}}\boldsymbol\epsilon$ \hfill (reuse $\boldsymbol\epsilon$)
\STATE Compute $\hat{\mathbf z}_0(t_k)\leftarrow \frac{1}{\sqrt{\bar\alpha_{t_k}}}\big(\mathbf z_{t_k}-\sqrt{1-\bar\alpha_{t_k}}\,\epsilon_\theta(\mathbf z_{t_k},\mathbf c,t_k)\big)$
\STATE Decode increment $\hat b_k \leftarrow \Delta\cdot\sigma(\mathrm{MLP}_k(\hat{\mathbf z}_0(t_k)))$ \hfill (Eq.~\ref{eq:ddp-decoding})
\ENDFOR
\STATE Compute $\mathcal L_{\mathrm{mia}}\leftarrow \sum_{k=1}^S \|b_k-\hat b_k\|_2^2$ and total loss $\mathcal L\leftarrow \mathcal L_{\mathrm{diff}}+\lambda\mathcal L_{\mathrm{mia}}$
\STATE \textbf{Inference:} run the reverse denoising chain; cache $\hat{\mathbf z}_0(t_k)$ at aligned timesteps; decode $\hat b_k$; apply truncation (Eq.~\ref{eq:usable-increments}) to obtain $\hat y$
\end{algorithmic}
\end{algorithm}

\subsection{Order-preserving Aggregation and Inference}
\label{sec:inference}
In inference, DiffoR runs the reverse process and obtains the aligned denoised estimates $\{\hat{\mathbf z}_0(t_k)\}_{k=1}^S$ along the denoising trajectory, then decodes $\hat{\mathbf B}$ using \eqref{eq:ddp-decoding}.
The final ordinal prediction is obtained via an additive aggregation with deterministic truncation:
\begin{align}
\tilde b_1 &:= \min\{\hat b_1,\Delta\}, \nonumber\\
\tilde b_k &:= \min\Big\{\hat b_k,\Delta,\big[1-\sum_{j=1}^{k-1}\tilde b_j\big]_+\Big\},\qquad k\ge 2, \label{eq:usable-increments}\\
\hat y &:= \sum_{k=1}^S \tilde b_k, \nonumber
\end{align}
where $[\cdot]_+=\max\{\cdot,0\}$. Intuitively, the nested $\min\{\cdot\}$ enforces feasibility of the increment representation by clipping each predicted increment to both its interval range ($[0,\Delta]$) and the remaining ordinal capacity, ensuring $\tilde b_k\in[0,\Delta]$ and $\hat y\in[0,1]$. This mapping prevents later increments from exceeding the remaining ordinal capacity.
Note that this dual-decoupled formulation is not merely heuristic. Sec.~\ref{sec:theory-main} shows that spatial decoupling can mitigate representation collapse and tighten ordinal error bounds, while temporal decoupling via multi-time inference yields non-worse (and potentially tighter) ordinal guarantees.

\begin{table*}[!ht]
\centering
\caption{Results of Image Aesthetics Assessment on four datasets.}
\label{tab:IAA}
\renewcommand{\arraystretch}{0.7}
\resizebox{0.95\textwidth}{!}{
\begin{tabular}{ll|cccccccc}
\toprule
Dataset & Metric
& RAPID~\citep{lu2014rapid}
& AADB~\citep{kong2016photo}
& PAM~\citep{ren2017personalized}
& NIMA~\citep{talebi2018nima}
& TANet~\citep{he2022rethinking}
& Mamba~\citep{gao2024aesmamba}
& GoR~\citep{ma2026gor}
& \textbf{DiffoR} \\
\midrule

\multirow{4}{*}{ICAA17K}
& MAE $\downarrow$
& 0.742 & 0.714 & 0.707 & 0.696 & 0.679 & 0.613 & \underline{0.584} & \textbf{0.552} \\
& XAUC $\uparrow$
& 0.642 & 0.666 & 0.673 & 0.684 & 0.701 & 0.766 & \underline{0.791} & \textbf{0.841} \\
& LCC $\uparrow$
& 0.516 & 0.531 & 0.539 & 0.546 & 0.559 & 0.614 & \underline{0.682} & \textbf{0.759}\\
& SRCC $\uparrow$
& 0.508 & 0.520 & 0.525 & 0.533 & 0.547 & 0.629 & \underline{0.679} & \textbf{0.765} \\

\midrule
\multirow{4}{*}{TAD66K}
& MAE $\downarrow$
& 1.766 & 1.463 & 1.314 & 1.422 & 1.081 & 1.035 & \underline{0.996} & \textbf{0.914} \\
& XAUC $\uparrow$
& 0.510 & 0.523 & 0.534 & 0.511 & 0.649 & 0.666 & \underline{0.677} & \textbf{0.694} \\
& LCC $\uparrow$
& 0.332 & 0.400 & 0.440 & 0.405 & 0.452 & 0.482 & \underline{0.541} & \textbf{0.553}\\
& SRCC $\uparrow$
& 0.314 & 0.379 & 0.422 & 0.390 & 0.428 & 0.468 & \underline{0.513} & \textbf{0.526}\\

\midrule
\multirow{4}{*}{AVA}
& MAE $\downarrow$
& 0.978 & 0.784 & 0.614 & 0.715 & 0.577 & 0.522 & \underline{0.395} & \textbf{0.388} \\
& XAUC $\uparrow$
& 0.513 & 0.534 & 0.619 & 0.532 & 0.659 & 0.697 & \underline{0.751} & \textbf{0.766} \\
& LCC $\uparrow$
& 0.336 & 0.431 & 0.531 & 0.472 & 0.568 & 0.663 & \underline{0.726} & \textbf{0.734} \\
& SRCC $\uparrow$
& 0.327 & 0.408 & 0.521 & 0.447 & 0.554 & 0.656 & \underline{0.701} & \textbf{0.725} \\

\midrule
\multirow{4}{*}{SPAQ}
& MAE $\downarrow$
& 1.089 & 1.083 & 1.073 & 1.076 & 1.047 & 0.988 & \underline{0.872} & \textbf{0.768} \\
& XAUC $\uparrow$
& 0.700 & 0.704 & 0.710 & 0.708 & 0.728 & 0.752 & \underline{0.765} & \textbf{0.805} \\
& LCC $\uparrow$
& 0.657 & 0.665 & 0.669 & 0.671 & 0.684 & 0.726 & \underline{0.743} & \textbf{0.799} \\
& SRCC $\uparrow$
& 0.613 & 0.616 & 0.622 & 0.620 & 0.638 & 0.690 & \underline{0.723} & \textbf{0.785} \\

\bottomrule
\end{tabular}
}
\begin{tablenotes}
\footnotesize
\item[*] The best and second best results are marked in \textbf{bold} and \underline{underline}, respectively.
$\uparrow$ indicates that higher values are better, while $\downarrow$ indicates the opposite.
\end{tablenotes}
\end{table*}

\subsection{Optimization Objective}
DiffoR is trained by jointly optimizing the diffusion denoising objective and the multi-scale increment supervision.
\subsubsection{Latent manifold (diffusion) loss.}
We adopt the standard noise-prediction objective:
\begin{equation}
\mathcal{L}_{\mathrm{diff}}
\;:=
\mathbb E_{t,\boldsymbol\epsilon}\Big[\big\|\boldsymbol\epsilon-\epsilon_\theta(\mathbf z_t,\mathbf c,t)\big\|_2^2\Big],
\end{equation}
where $t$ is sampled from a pre-defined distribution over $\{1,\dots,T\}$. To better match DDP, we use a mixture that emphasizes the aligned timesteps $\{t_k\}$ while still covering the full trajectory.

\subsubsection{Multi-scale aggregation loss.}
To supervise the decoded multivariate representation, we define a scale-wise regression objective:
\begin{equation}
\mathcal{L}_{\mathrm{mia}}
\;:=
\sum_{k=1}^{S}
\big\| b_k - \hat{b}_k \big\|_2^2,
\end{equation}
encouraging different heads to specialize in distinct ordinal scales (coarse-to-fine).

\subsubsection{Overall objective.} 
The whole training objective is
\begin{equation}
\mathcal{L}
\;:=
\mathcal{L}_{\mathrm{diff}}
+
\lambda\,\mathcal{L}_{\mathrm{mia}},
\label{eq:total_loss}
\end{equation}
where $\lambda$ controls the trade-off between denoising quality and multi-scale ordinal supervision.

The above design directly supports the statements in the Introduction section and our contributions: spatial decoupling (MIA) improves representational capacity for hierarchical ordinal semantics, while temporal decoupling (DDP) provides multi-time, frequency-aligned denoising states for coarse-to-fine calibration. Sec.~\ref{sec:theory-main} formalizes how these mechanisms translate into tighter ordinal error bounds and improved optimization behavior.

\subsection{Theoretical Results}
\label{sec:theory-main}
We provide theoretical results connecting spatio-temporal decoupling to \emph{ordinal prediction quality}. For simplicity we present the scalar case (extension to vectors follows by replacing squares with squared norms). We establish: (i) spatial decoupling mitigates representation collapse and tightens the ordinal bound; (ii) temporal decoupling yields a non-worse ordinal bound via multi-time selection/fusion; (iii) MH-MT admits linear convergence under standard smoothness and PL conditions. Full proofs are in Appendix~\ref{app:paradigms}.


\begin{theorem}[Spatial decoupling mitigates collapse and tightens the ordinal bound]
\label{thm:spatial-decoupling-main}
Consider two representations over the same samples: a single-embedding representation matrix $U^{\mathrm{SE}}$ and a multi-embedding (concatenated) representation matrix $U^{\mathrm{ME}}$ with the \emph{same total dimension} (e.g., $4\times 32$ vs.\ $1\times 128$).
If the multi-embedding representation spans a larger subspace,
\[
\mathrm{col}(U^{\mathrm{SE}})\subseteq \mathrm{col}(U^{\mathrm{ME}}),
\]
then the best linear readout risk (least-squares) for predicting the continuous target $z_0$ is non-increasing:
\[
\mathcal R(U^{\mathrm{ME}})\le \mathcal R(U^{\mathrm{SE}}),
\qquad
\mathcal R(U)\triangleq \min_w \frac{1}{n}\|Uw-z_0\|_2^2.
\]
If the inclusion is strict and the target contains a component explainable by $\mathrm{col}(U^{\mathrm{ME}})$ but not by $\mathrm{col}(U^{\mathrm{SE}})$, then $\mathcal R(U^{\mathrm{ME}})< \mathcal R(U^{\mathrm{SE}})$.
Consequently, via the ordinal MSE-to-error inequality, improving representation fit tightens the ordinal error upper bound.
\end{theorem}
\noindent\textbf{Proof sketch.}
Least squares is orthogonal projection; projecting onto a larger subspace cannot increase residual norm (strictly decreases when the residual becomes representable).

\begin{theorem}[Temporal decoupling improves the best achievable ordinal-error bound]
\label{thm:temporal-decoupling-main}
Let $\hat z_0(t)$ be the diffusion-derived estimator constructed from a noise predictor at time $t$ (Appendix, Eq.~(9)).
Assume threshold decoding with minimum gap $\gamma>0$ (Appendix, Assumption~\ref{ass:thresholds}) and the standard diffusion forward process.
Then for any fixed time $t$,
\[
\mathbb P(\hat y(t)\ne y)\ \le\ \frac{4}{\gamma^2}\,\mathbb E[(\hat z_0(t)-z_0)^2],
\]
and the regression MSE satisfies the exact identity
\[
\mathbb E[(\hat z_0(t)-z_0)^2]
:=\frac{1-\bar{\alpha}_t}{\bar{\alpha}_t}\,\mathbb E[(\hat\epsilon(t)-\epsilon)^2].
\]
Moreover, given a set of time points $\mathcal T=\{t_1,\dots,t_M\}$, define the oracle-selected predictor
$t^\dagger\in\arg\min_{t\in\mathcal T}\mathbb E[(\hat z_0(t)-z_0)^2]$.
Then its ordinal error bound is no worse than any fixed single-time baseline $t_\star\in\mathcal T$:
\[
\mathbb P(\hat y(t^\dagger)\ne y)\ \le\ \frac{4}{\gamma^2}\,\mathbb E[(\hat z_0(t^\dagger)-z_0)^2]
\ \le\ \frac{4}{\gamma^2}\,\mathbb E[(\hat z_0(t_\star)-z_0)^2].
\]
\end{theorem}
\noindent\textbf{Proof sketch.}
The first inequality is the ordinal MSE-to-error bound; the MSE identity is a direct algebraic consequence of the diffusion reparameterization; the oracle bound follows from the minimizer definition.

\begin{theorem}[Linear convergence rate of DiffoR on the multi-time objective]
\label{thm:convergence-main}
Consider the multi-time training objective
\[
F_{\mathrm{MH}}(\vartheta)=\sum_{m=1}^M \lambda_m f_m(\theta_m),\qquad \vartheta=(\theta_1,\dots,\theta_M),\ \lambda_m>0,
\]
where each $f_m$ is differentiable.
Assume each $f_m$ is $L_m$-smooth and satisfies the Polyak--\L{}ojasiewicz (PL) inequality with constant $\mu_m>0$:
\[
\frac{1}{2}\|\nabla f_m(\theta)\|^2 \ge \mu_m\bigl(f_m(\theta)-f_m^\star\bigr),\qquad f_m^\star=\inf_\theta f_m(\theta).
\]
Let $L\triangleq \sum_{m=1}^M \lambda_m L_m$, $\mu_{\min}\triangleq \min_m \mu_m$, and $\lambda_{\min}\triangleq \min_m \lambda_m$.
Then gradient descent with step size $\eta\le 1/L$ yields the linear rate
\[
F_{\mathrm{MH}}(\vartheta^{k})-F_{\mathrm{MH}}^\star
\le
\bigl(1-\eta\,\mu_{\min}\lambda_{\min}\bigr)^k\bigl(F_{\mathrm{MH}}(\vartheta^{0})-F_{\mathrm{MH}}^\star\bigr),
\]
\end{theorem}
where
$F_{\mathrm{MH}}^\star=\sum_{m=1}^M \lambda_m f_m^\star.$

\noindent\textbf{Proof sketch.}
Using smoothness, one obtains a standard descent inequality
$F(\vartheta^{k+1})\le F(\vartheta^{k})-\frac{\eta}{2}\|\nabla F(\vartheta^{k})\|^2$ for $\eta\le 1/L$.
For MH-MT, the gradient is block-separated across heads, so $\|\nabla F_{\mathrm{MH}}\|^2=\sum_m \lambda_m^2\|\nabla f_m(\theta_m)\|^2$ with no cross terms; applying per-task PL and $\lambda_m^2\ge \lambda_{\min}\lambda_m$ yields
$\frac{1}{2}\|\nabla F_{\mathrm{MH}}\|^2\ge \mu_{\min}\lambda_{\min}(F_{\mathrm{MH}}-F_{\mathrm{MH}}^\star)$, giving the stated recursion.

\begin{remark}[Interpretation and connection to our model design]
\label{rem:interpretation-main}
Theorem~\ref{thm:spatial-decoupling-main} formalizes the benefit of \textbf{spatial decoupling}: splitting the representation into multiple embedding sets promotes \emph{subspace diversification} (less collapse), which improves the best-fit continuous regression error and therefore tightens the ordinal bound.
Theorem~\ref{thm:temporal-decoupling-main} formalizes the benefit of \textbf{temporal decoupling}: multi-time, time-specific modeling provides a strictly richer inference strategy space (selection/fusion) and yields a non-worse ordinal bound than any single fixed time step.
Theorem~\ref{thm:convergence-main} complements the above with a \textbf{trainability} guarantee: parameter separation across time steps makes the multi-time objective inherit a global PL-type descent behavior from per-time objectives, yielding an explicit linear convergence rate under standard conditions.
Together, these results justify why the spatio-temporal decoupled generative ordinal regression architecture is beneficial: it improves representational capacity, inference flexibility, and optimization robustness in ways that provably translate into tighter ordinal prediction guarantees.
\end{remark}


\begin{table*}[t]
  \centering
  \caption{Facial age estimation results on four benchmarks}
  \label{tab:age}
  \vspace{-1em}
  \resizebox{0.8\linewidth}{!}{
  \renewcommand{\arraystretch}{0.85}
    \begin{tabular}{ccccccccc}
      \toprule
      \multirow{2}{*}{Method}
        & \multicolumn{2}{c}{UTKFace}
        & \multicolumn{2}{c}{FG-NET}
        & \multicolumn{2}{c}{MORPH}
        & \multicolumn{2}{c}{CACD} \\
      \cmidrule(lr){2-3} \cmidrule(lr){4-5}
      \cmidrule(lr){6-7} \cmidrule(lr){8-9}
      & MAE $\downarrow$ & CS(\%) $\uparrow$ 
      & MAE $\downarrow$ & CS(\%) $\uparrow$ 
      & MAE $\downarrow$ & CS(\%) $\uparrow$ 
      & MAE $\downarrow$ & CS(\%) $\uparrow$ \\
      \midrule
      OR-CNN~\citep{OR-CNN}
        & 4.40 & 63.67 & 5.09 & 83.80 & 2.83 & 61.97 & 4.01 & 73.41 \\
      DLDL~\citep{DLDL}
        & 4.39 & 63.65 & 5.26 & 83.83 & 2.81 & 62.43 & 3.96 & 73.37 \\
      SORD~\citep{SORD}
        & 4.36 & 64.25 & 5.59 & 82.83 & 2.81 & 61.31 & 3.96 & 73.48 \\
      Mean-Var.~\citep{meanvar}
        & 4.42 & 63.36 & 5.45 & 83.43 & 2.83 & 62.87 & 4.07 & 72.98 \\
      Unimodal~\citep{Unimodal}
        & 4.47 & 62.67 & 5.13 & 83.97 & 2.78 & 63.15 & 4.10 & 73.55 \\
      FaRL~\citep{FaRL}
        & 3.87 & 65.38 & 4.95 & 84.52 & 3.04 & 63.49 & 3.96 & 74.18 \\
      GoR~\citep{ma2026gor}
        & \underline{3.43} & \underline{66.58} & \underline{4.68} & \underline{85.66} & \underline{2.69} & \underline{64.95} & \underline{3.73} & \underline{75.29} \\
      \textbf{DiffoR}
        & \textbf{3.39} & \textbf{67.23} & \textbf{4.29} & \textbf{87.35} & \textbf{2.58} & \textbf{66.79} & \textbf{3.64} & \textbf{76.88} \\
      \bottomrule
    \end{tabular}
  }
\end{table*}

\section{Experiments}
In this section, we first evaluate DiffoR’s overall performance across diverse domains, followed by an in-depth analysis of architectural choices, distributional visualizations, and component ablations to elucidate its underlying mechanisms.
We employ a comprehensive set of metrics, including Mean Absolute Error (MAE), Cumulative Score (CS), XAUC~\citep{d2q}, Linear Correlation Coefficient (LCC), and Spearman’s Rank Correlation Coefficient (SRCC). Due to space limit, detailed metric definitions, implementation specifics, and supplementary results are provided in Appendix~\ref{app:experiment}.

\subsection{Overall Performance across Domains}
\subsubsection{Image Aesthetics Assessment (IAA)}
\label{sec:iaa}
\paragraph{Setting.}
Following the protocol in~\citep{he2023thinking}, we benchmark DiffoR against 14 representative baselines on four standard datasets: AVA~\citep{murray2012ava}, TAD66K~\citep{he2022rethinking},  ICAA17K~\citep{he2023thinking}, and SPAQ~\citep{fang2020perceptual}. Performance is assessed using MAE, XAUC, LCC, and SRCC. 
Due to space limit, here we present only the results of six top-performing methods in Tab.~\ref{tab:IAA}. The complete results are in Appendix~\ref{app:IAA}.
For all experiments, we employ a ResNet50~\citep{he2016deep} backbone as the encoder.

\subsubsection{Watch Time Prediction (WTP)}
\begin{table}[t]
\centering
\caption{Performance comparison among different approaches on KuaiRec and KuaiRand.}
\label{tab:watch_time}
\resizebox{0.46\textwidth}{!}{%
\renewcommand{\arraystretch}{0.85}
\begin{tabular}{c|cc|cc}
\toprule
\multirow{2}{*}{Method}
& \multicolumn{2}{c|}{KuaiRec}
& \multicolumn{2}{c}{KuaiRand} \\
\cmidrule(lr){2-3}\cmidrule(lr){4-5}
& MAE~$\downarrow$ & XAUC~$\uparrow$
& MAE~$\downarrow$ & XAUC~$\uparrow$ \\
\midrule


D2Co~\cite{d2co}
& 3.2633 & 0.5895
& 20.7854 & 0.6547 \\

CWM~\cite{cwm}
& 3.3532 & 0.5899
& 19.6351 & 0.6668 \\

D2Q~\cite{d2q}
& 3.2696 & 0.6043
& 19.4258 & \underline{0.6715} \\

TPM~\cite{tpm}
& 3.4584 & 0.5819
& 22.5950 & 0.6303 \\

CREAD~\cite{cread}
& 3.2290 & 0.6119
& 19.8087 & 0.6678 \\

PTPM~\cite{PTPM}
& 3.2865 & 0.6033
& 20.6584 & 0.6679 \\

SWaT~\cite{swat}
& 3.3496 & 0.5888
& 22.3353 & 0.6515 \\

GoR~\cite{ma2026gor}
& 3.1985 & \underline{0.6117}
& \underline{19.2742} & 0.6682 \\

EGMN~\cite{egmn}
& \underline{3.1890} & 0.6098
& 19.3246 & 0.6682 \\

\textbf{DiffoR}
& \textbf{3.1427} & \textbf{0.6126}
& \textbf{19.1066} & \textbf{0.6763} \\

\bottomrule
\end{tabular}%
}
\vspace{-1em}
\end{table}

\paragraph{Performance.}
Tab.~\ref{tab:IAA} highlights DiffoR's dominance across all metrics.
Even with a generic ResNet50, DiffoR surpasses all SOTA methods equipped with specialized architectures by a significant margin in both ordinal-sensitive and accuracy metrics.
Considering the pivotal role of visual features in IAA~\citep{he2022rethinking}, our continuous generative paradigm not only excels in ordinal modeling but also demonstrates robust encoder-agnostic generalization.

\subsubsection{Facial Age Estimation (FAE)}
\paragraph{Setting.}
FAE aims to predict chronological age from facial imagery by analyzing visual cues. 
Adhering to the evaluation protocol established in~\citep{FaRL}, we benchmark DiffoR on four standard datasets: UTKFace~\citep{zhang2017age}, FG-NET~\citep{993553}, MORPH~\citep{ricanek2006morph}, and CACD~\citep{chen2014cross}, using Mean Absolute Error (MAE) and Cumulative Score (CS) with a tolerance level of $L=5$ as metrics, in comparison with 7 SOTAs including OR-CNN~\citep{OR-CNN}, DLDL~\citep{DLDL}, SORD~\citep{SORD}, Mean-Var.~\citep{meanvar}, Unimodal~\citep{Unimodal}, FaRL~\citep{FaRL}, and GoR~\citep{ma2026gor}. For all experiments, we employ a standard ResNet50 backbone as the encoder.

\paragraph{Performance.}
As presented in Tab.~\ref{tab:age}, DiffoR achieves SOTA performance, significantly outperforming all baselines across all datasets and metrics. Specifically, DiffoR delivers substantial gains, with MAE reductions ranging from 4.08\% (MORPH) to 8.33\% (FG-NET) and CS improvements between 1.59\% (CACD) and 0.65\% (UTKFace).
These consistent improvements underscore DiffoR's robust generalization capability across diverse ordinal distributions.

\paragraph{Setting.}
We conduct evaluations on two publicly available datasets KuaiRec~\citep{gao2022kuairec} and KuaiRand~\citep{gao2022kuairand}, which are both collected from real-world recommender platforms.
Using the Feed-Forward Network (FFN) as an encoder~\citep{ma2024generative}, we report MAE and XAUC.

\paragraph{Performance.}
As shown in Tab.~\ref{tab:watch_time}, DiffoR surpasses all nine baselines. Against the runner-up, it reduces MAE by up to 1.45\% on KuaiRec.
This dominance extends to the KuaiRand, where DiffoR achieves a 0.168 MAE drop and a 0.71\% XAUC gain.
These results not only validate the method's efficacy, but also highlight its practical value for optimizing real-world recommendation systems.

\subsubsection{Life Time Value Prediction (LTV)}
\begin{table}[t]
  \centering
  \caption{Performance comparison on LTV datasets.}
  \vspace{-1em}
  \resizebox{0.94\linewidth}{!}{
  \renewcommand{\arraystretch}{0.9}
    \begin{tabular}{c|cc|cc}
      \toprule
      \multirow{2}{*}{Method} & \multicolumn{2}{c|}{Criteo-SSC} & \multicolumn{2}{c}{Kaggle} \\
      & MAE~$\downarrow$ & SRCC~$\uparrow$ & MAE~$\downarrow$ & SRCC~$\uparrow$ \\
      \midrule
      Two-stage~\citep{drachen2018or} & 21.719 & 0.2386 & 74.782 & 0.431 \\
      MTL-MSE~\citep{ma2018entire}    & 21.190 & 0.2478 & 74.065 & 0.433 \\
      ZILN~\citep{wang2019deep}       & 20.880 & 0.2434 & 72.528 & 0.524 \\
      MDME~\citep{li2022billion}      & 16.598 & 0.2269 & 72.900 & 0.516 \\
      MDAN~\citep{liu2024mdan}        & 20.030 & 0.2470 & 73.940 & 0.437 \\
      OptDist~\citep{weng2024optdist} & 15.784 & 0.2505 & 70.929 & 0.525 \\
      HiLTV~\citep{xu2025hiltv} & 14.764 & 0.2645 & 69.331 & 0.512 \\
      GoR~\citep{ma2026gor} & \underline{12.996} & \underline{0.3026} & \underline{67.035} & \underline{0.533} \\
      \textbf{DiffoR} & \textbf{12.533} & \textbf{0.3066} & \textbf{66.522} & \textbf{0.557} \\
      \bottomrule
    \end{tabular}
  }
  \vspace{-1em}
  \label{tab:ltv}
\end{table}
\paragraph{Setting.}
We benchmark DiffoR on the Criteo-SSC and Kaggle datasets following~\citep{weng2024optdist}, reporting MAE and SRCC.
An FFN-based encoder (identical to the WTP task) is employed.
Detailed specifications of the encoder, baselines, and dataset statistics are provided in Appendix~\ref{app:ltv}.

\paragraph{Performance.}
As presented in Tab.~\ref{tab:ltv}, DiffoR consistently outperforms eight existing methods across both ordinal (SRCC) and numeric (MAE) metrics. 
On Criteo-SSC, DiffoR surpasses GoR, reducing MAE by 3.56\% and improving SRCC by 1.32\%.
on Kaggle, DiffoR achieves a 0.531 reduction in MAE and a 4.5\% boost in SRCC compared to GoR, substantiating the superiority of DiffoR.

\subsection{Further Analysis}
\subsubsection{Ablation Study}
\paragraph{Module Contribution}
We investigate the effects of our core components: Multi-scale Increment Aggregation (MIA) and Dynamic Denoising Perception (DDP).
As shown in Tab.~\ref{tab:ablation}, removing MIA (Row (d)) leads to a sharp performance drop (e.g., SRCC $\downarrow$ 0.1), highlighting its necessity in capturing hierarchical ordinal dependencies.
Similarly, discarding DDP (Row (e)) degrades performance, confirming the importance of synchronizing denoising steps with feature frequencies.
The removal of both modules results in the worst performance (Row (f)), verifying that MIA and DDP function synergistically to enable robust learning.

\paragraph{Architecture Analysis}
To validate the universality of our paradigm, we replace the diffusion backbone with Flow Matching (Row (b) in Tab.~\ref{tab:ablation}). The comparable performance confirms that the efficacy stems from the Continuous Generative Ordinal Regression formulation itself, rather than a specific generative model.
Furthermore, replacing the Transformer architecture with a simple MLP-based structure (Row (c) in Tab.~\ref{tab:ablation}) causes a performance dip due to the loss of attention-based feature interaction. However, this variant still outperforms the SOTA baselines, demonstrating that our generative framework provides a fundamental improvement robust to architectural simplifications.

\begin{table}[t]
    \centering
    \caption{Ablation study of key modules and generative architectures on the ICAA17K dataset under the IAA task.}
    \label{tab:ablation}
    \setlength{\tabcolsep}{3.5pt}
    \renewcommand{\arraystretch}{1.0}
    \begin{tabular}{cc|cccc}
        \toprule
        \multicolumn{2}{c|}{Variant}
        & MAE $\downarrow$ & XAUC $\uparrow$ & LCC $\uparrow$ & SRCC $\uparrow$ \\
        \midrule
        (a) & \textbf{DiffoR}
        & \textbf{0.552} & \textbf{0.841} & \textbf{0.759} & \textbf{0.765} \\
        \midrule
        \multicolumn{6}{l}{\textit{Generative Architecture Ablation}} \\
        (b) & w/ Flow Matching
        & 0.551 & 0.842 & 0.758 & 0.765 \\
        (c) & Rep. Trans. w/ MLP
        & 0.568 & 0.805 & 0.710 & 0.728 \\
        \midrule
        \multicolumn{6}{l}{\textit{Key Modules Ablation}} \\
        (d) & w/o \textit{MIA}
        & 0.597 & 0.786 & 0.675 & 0.666 \\
        (e) & w/o \textit{DDP}
        & 0.566 & 0.801 & 0.696 & 0.701 \\
        (f) & w/o Both line (d) \& (e)
        & 0.606 & 0.763 & 0.653 & 0.642 \\
        \bottomrule
    \end{tabular}
\end{table}

\subsubsection{Latent Space Visualization}
\begin{figure}[t]
    \centering
    \includegraphics[width=0.6\linewidth]{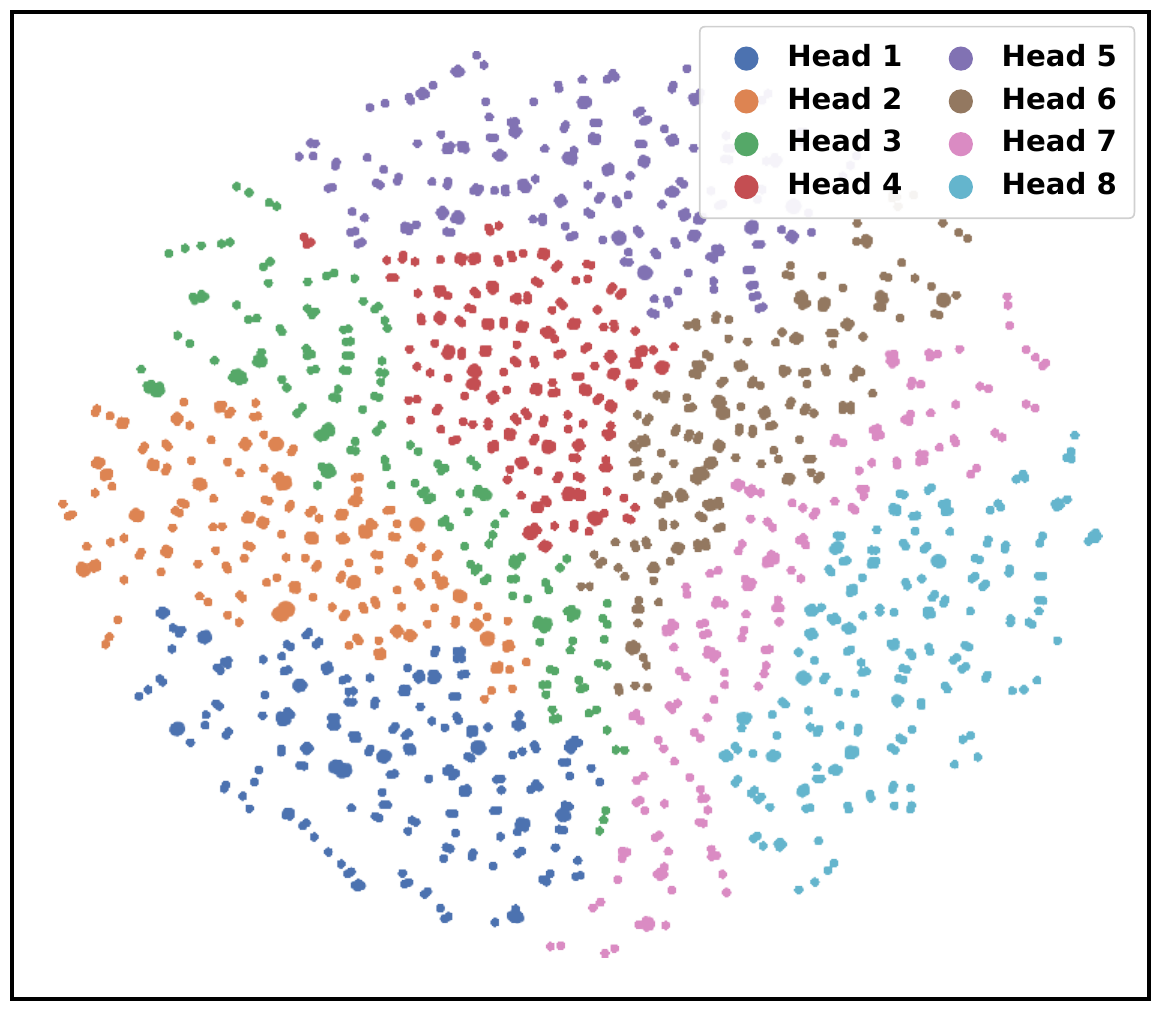}
    \caption{t-SNE visualization of embeddings from different attention heads. The distinct clusters demonstrate that each head captures diverse and non-redundant ordinal patterns.}
    \label{fig:tsne}
\end{figure}
To see whether DiffoR effectively decomposes the ordinal regression task, we visualize the latent representations learned by different attention heads using t-SNE. As illustrated in Fig.~\ref{fig:tsne}, the embeddings from the 8 heads exhibit clear spatial disentanglement, forming distinct and cohesive clusters without significant overlap. This implies that each head specializes in capturing a specific ordinal pattern or semantic subspace. Such diversity confirms that our Multi-scale Increment Aggregation strategy successfully encourages the model to learn complementary features, rather than collapsing into redundant representations.

\subsubsection{Hyperparameter Effect}
\begin{figure}[t]
    \centering
    \includegraphics[width=0.86\linewidth]{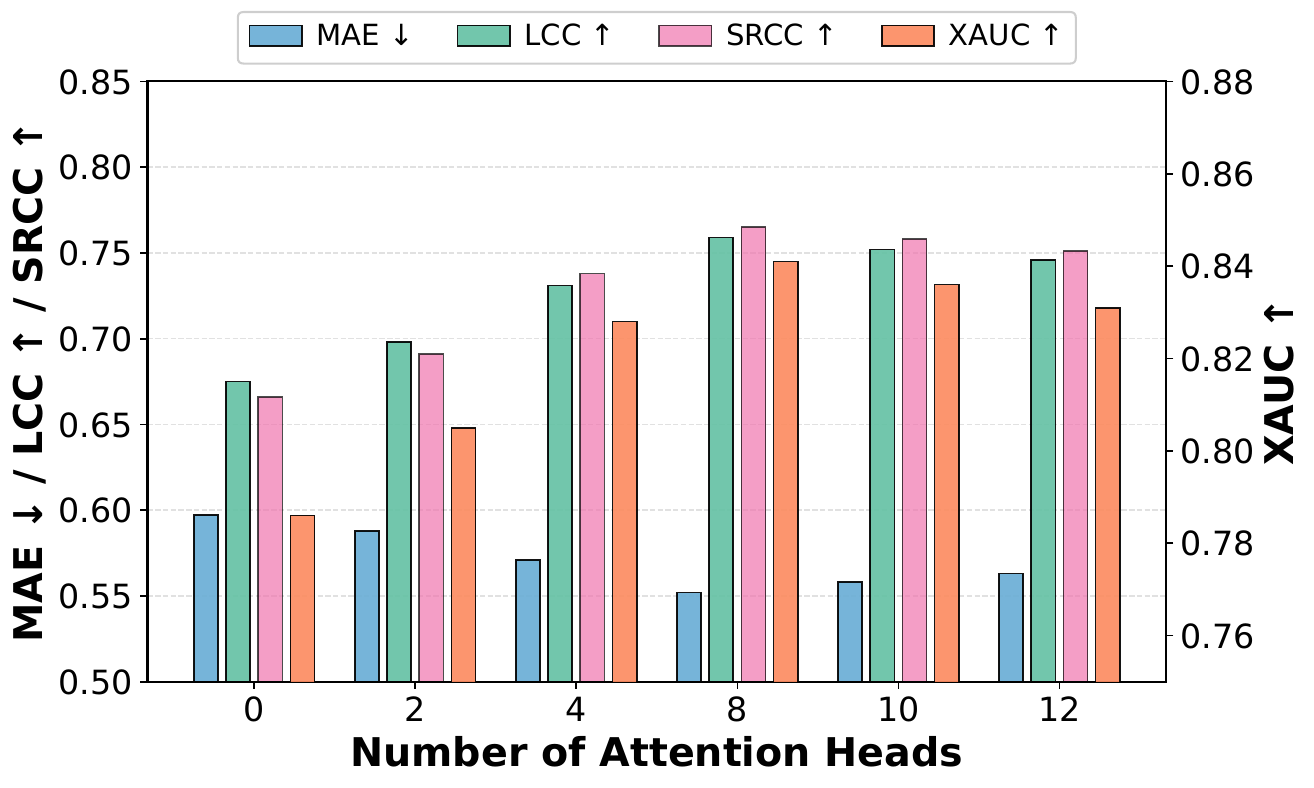}
    \caption{Impact of the number of attention heads on ICAA17K. ``0 head'' corresponds to the variant without MIA.}
    \label{fig:head_ablation}
    \vspace{-1em}
\end{figure}
We further investigate the impact of the number of attention heads on the ICAA17K dataset, which governs the granularity of ordinal subspace decoupling. As shown in Fig.~\ref{fig:head_ablation}, the configuration with 0 head (equivalent to removing MIA, cf. Tab.~\ref{tab:ablation} Row (d)).
As the number of heads increases, we observe a steady improvement across all metrics, peaking at 8 heads. 
This suggests that a sufficient number of parallel subspaces is crucial for capturing diverse ordinal increments. 
However, further increasing the heads (to 10 or 12) leads to a slight performance degradation, likely due to over-parameterization or the fragmentation of semantic features into overly fine-grained, noisy subspaces.

\section{Conclusion}
This work addresses the fundamental bottlenecks of discretization in Ordinal Regression (OR) by proposing DiffoR, a novel continuous generative framework.
By reformulating OR as a conditional value recovery process via diffusion, DiffoR bypasses quantization artifacts and natively captures the non-stationary semantic transitions of ordinal data.
Our proposed Dual-Decoupling Strategy --- synergizing spatial increment aggregation with temporal denoising perception --- ensures the preservation of hierarchical ordinal topology.
Extensive experiments across 12 benchmarks in four domains demonstrate DiffoR's consistent superiority and robust generalization.
By establishing a high-precision, universal solution for OR, this work offers a robust foundation for future research in modeling data with inherent ordering.

\section*{Acknowledgments} 
This work was partially supported by Kuaishou Technology. Shuigeng Zhou was supported by National Social Science Fund of China (NSFC) under grant No. 24\&ZD185. The computations in this research were performed using the CFFF platform of Fudan University.


\bibliographystyle{ACM-Reference-Format}
\bibliography{main}

\newpage
\appendix
\section{Limitaions of Existing Modeling Paradigms}
\label{app:paradigms}

This appendix provides a self-contained theoretical treatment of existing modeling paradigms for watch-time prediction (WTP) and includes full proofs for the propositions stated in the main paper. Throughout, $\mathbf{x}$ denotes the conditioning features, $y \in \mathbb{R}^+$ denotes watch time, and $P_{\text{data}}(\cdot)$ denotes the true (unknown) data-generating distribution.

\subsection{Mean Collapse in Conventional Regression}
\label{app:meancollapse}

A common baseline is point-wise regression trained with mean squared error (MSE):
\[
\min_f \; \mathbb{E}_{(\mathbf{x},y)\sim P_{\text{data}}}\big[(y-f(\mathbf{x}))^2\big].
\]
When $P_{\text{data}}(y\mid \mathbf{x})$ is multimodal, the MSE objective drives the predictor to the conditional mean, which can lie in a low-density region.

\begin{proposition}[Mean Collapse Effect]
\label{pro:meanCollapse_app}
Assume that, for a fixed $\mathbf{x}$, $P_{\text{data}}(y\mid \mathbf{x})$ consists of multiple well-separated modes with separation $\Delta$ and within-mode scale $\sigma$. Let $f^*(\mathbf{x})=\mathbb{E}[y\mid \mathbf{x}]$ denote the MSE-optimal regressor. Then, under the disjoint-modes regime $\Delta/\sigma\to\infty$,
\begin{equation}
\lim_{\Delta/\sigma\to\infty}
P_{\text{data}}\!\left(f^*(\mathbf{x}) \mid \mathbf{x}\right)=0.
\end{equation}
\end{proposition}

\begin{proof}
\textbf{Step 1: MSE-optimal predictor.}
For any measurable $f$, the MSE risk is
\[
\mathcal{R}(f)=\mathbb{E}\big[(y-f(\mathbf{x}))^2\big].
\]
Conditioning on $\mathbf{x}$ and minimizing pointwise gives
\[
f^*(\mathbf{x})
=
\arg\min_{a\in\mathbb{R}}
\mathbb{E}\big[(y-a)^2\mid \mathbf{x}\big]
=
\mathbb{E}[y\mid \mathbf{x}].
\]

\textbf{Step 2: Density at the conditional mean under separated modes.}
Consider the bimodal case for clarity:
\[
P_{\text{data}}(y\mid \mathbf{x})
=
\frac{1}{2}\mathcal{N}(y;\mu_1,\sigma^2)
+
\frac{1}{2}\mathcal{N}(y;\mu_2,\sigma^2),
\quad
\Delta\triangleq|\mu_1-\mu_2|.
\]
Then
\[
f^*(\mathbf{x})=\bar{\mu}\triangleq\frac{\mu_1+\mu_2}{2}.
\]
The conditional density at $\bar{\mu}$ is
\begin{align}
P_{\text{data}}(\bar{\mu}\mid \mathbf{x})
&=
\frac{1}{2}\mathcal{N}(\bar{\mu};\mu_1,\sigma^2)
+
\frac{1}{2}\mathcal{N}(\bar{\mu};\mu_2,\sigma^2) \nonumber\\
&=
\frac{1}{\sqrt{2\pi}\sigma}
\exp\!\left(-\frac{\Delta^2}{8\sigma^2}\right).
\end{align}
Taking the limit $\Delta/\sigma\to\infty$ yields
\[
\lim_{\Delta/\sigma\to\infty}P_{\text{data}}(\bar{\mu}\mid \mathbf{x})=0,
\]
which proves the claim. The argument extends to multiple modes: when modes are mutually separated, the conditional mean lies between modes and its density decays exponentially with separation.
\end{proof}

\subsection{Limitations of Discretization}
\label{app:ordinal}

Ordinal regression methods (e.g., CREAD~\cite{cread} and SWaT~\cite{swat}) discretize the continuous watch time $y$ using predefined thresholds
\[
c_1 < c_2 < \dots < c_M,
\]
and transform $y$ into an ordered sequence of binary decisions:
\[
\mathbf{B}^m = \mathbb{I}(y > c_m), \quad m=1,\dots,M.
\]
This converts a continuous regression problem into multiple classification sub-tasks and can reduce sensitivity to long-tailed targets. However, it introduces: (i) \textbf{hard quantization error}, since thresholds break continuity and create boundary effects; and (ii) an \textbf{interval independence assumption}, since standard losses (e.g., cross-entropy) typically treat $\{\mathbf{B}^m\}$ as conditionally independent given $\mathbf{x}$, ignoring ordinal dependencies across intervals~\cite{ma2024generative}.

\begin{proposition}[Dependency Error in Discretized Modeling]
\label{pro:disError_app}
Let $P_{\text{data}}(\mathbf{B}\mid \mathbf{x})$ denote the true joint distribution of interval decisions, where $\mathbf{B}=(\mathbf{B}^1,\dots,\mathbf{B}^M)$, and let the naive discretization model assume conditional independence:
\[
P_{\text{naive}}(\mathbf{B}\mid \mathbf{x})=\prod_{m=1}^{M} P(\mathbf{B}^m\mid \mathbf{x}).
\]
Then the modeling error measured by KL divergence admits the decomposition
\begin{equation}
\begin{aligned}
D_{\mathrm{KL}}\!\left(P_{\text{data}} \,\|\, P_{\text{naive}}\right)
&=
\sum_{m=1}^{M}
\mathbb{E}_{\mathbf{B}^{<m} \sim P_{\text{data}}}
\!\left[
D_{\mathrm{KL}}
\!\left(
P(\mathbf{B}^m \mid \mathbf{x}, \mathbf{B}^{<m})
\,\|\, 
P(\mathbf{B}^m \mid \mathbf{x})
\right)
\right] \\
&=
\sum_{m=1}^{M}
I(\mathbf{B}^m;\mathbf{B}^{<m}\mid \mathbf{x}),
\end{aligned}
\end{equation}
where $\mathbf{B}^{<m}=(\mathbf{B}^1,\dots,\mathbf{B}^{m-1})$ and $I(\mathbf{B}^m;\mathbf{B}^{<m}\mid \mathbf{x})$ is the conditional mutual information.
\end{proposition}

\begin{proof}
We first write the true distribution using the chain rule:
\[
P_{\text{data}}(\mathbf{B}\mid \mathbf{x})
=
\prod_{m=1}^{M}
P(\mathbf{B}^m \mid \mathbf{x}, \mathbf{B}^{<m}).
\]
By definition,
\begin{align}
D_{\mathrm{KL}}\!\left(P_{\text{data}} \,\|\, P_{\text{naive}}\right)
&=
\mathbb{E}_{\mathbf{B}\sim P_{\text{data}}}
\left[
\log
\frac{P_{\text{data}}(\mathbf{B}\mid \mathbf{x})}{P_{\text{naive}}(\mathbf{B}\mid \mathbf{x})}
\right] \nonumber\\
&=
\mathbb{E}_{\mathbf{B}\sim P_{\text{data}}}
\left[
\log
\frac{\prod_{m=1}^{M} P(\mathbf{B}^m \mid \mathbf{x}, \mathbf{B}^{<m})}
{\prod_{m=1}^{M} P(\mathbf{B}^m \mid \mathbf{x})}
\right] \nonumber\\
&=
\sum_{m=1}^{M}
\mathbb{E}_{\mathbf{B}\sim P_{\text{data}}}
\left[
\log
\frac{P(\mathbf{B}^m \mid \mathbf{x}, \mathbf{B}^{<m})}{P(\mathbf{B}^m \mid \mathbf{x})}
\right]. \label{eq:kl_sum_log_ratio}
\end{align}
Taking expectation over $\mathbf{B}^{<m}$ explicitly yields
\begin{equation}
\sum_{m=1}^{M}
\mathbb{E}_{\mathbf{B}^{<m}\sim P_{\text{data}}}
\left[
\sum_{b\in\{0,1\}}
P(\mathbf{B}^m=b\mid \mathbf{x}, \mathbf{B}^{<m})
\log\frac{P(\mathbf{B}^m=b\mid \mathbf{x}, \mathbf{B}^{<m})}{P(\mathbf{B}^m=b\mid \mathbf{x})}
\right],
\end{equation}
which is exactly
\[
\sum_{m=1}^{M}
\mathbb{E}_{\mathbf{B}^{<m}\sim P_{\text{data}}}
\left[
D_{\mathrm{KL}}
\!\left(
P(\mathbf{B}^m \mid \mathbf{x}, \mathbf{B}^{<m})
\,\|\, 
P(\mathbf{B}^m \mid \mathbf{x})
\right)
\right].
\]
Finally, by the definition of conditional mutual information,
\begin{align}
I(\mathbf{B}^m;\mathbf{B}^{<m}\mid \mathbf{x})
&\triangleq
\mathbb{E}_{\mathbf{B}^{<m}\sim P_{\text{data}}}
\left[
D_{\mathrm{KL}}
\!\left(
P(\mathbf{B}^m \mid \mathbf{x}, \mathbf{B}^{<m})
\,\|\, 
P(\mathbf{B}^m \mid \mathbf{x})
\right)
\right],
\end{align}
so summing over $m$ gives the stated equality.
\end{proof}

\subsection{Proof of AR Limitations}
\label{app:ar_bottleneck}

We provide a bias variance decomposition of the expected squared regression error for tokenized AR models ~\cite{ma2024generative}.

\paragraph{Setup and Notation.}
Consider a continuous target reconstructed from a sequence of value tokens:
$y \triangleq \sum_{t=1}^{T} \phi(s^{t})$,
where $s^{t}$ denotes the ground-truth token at step $t$ and $\phi(\cdot)$ maps a token to its numeric value.
The AR model predicts a token sequence $\hat{s}^{1:T}$, yielding the prediction
$\hat{y} \triangleq \sum_{t=1}^{T} \phi(\hat{s}^{t})$.

Define the numeric token values
$C^{t} \triangleq \phi(s^{t}), \hat{C}^{t} \triangleq \phi(\hat{s}^{t})$,
and the step-wise prediction error
$\Delta_t \triangleq \hat{C}^{t} - C^{t}$. By assumption, all token values are bounded:
$C^{t}, \hat{C}^{t} \in [w_{\min}, w_{\max}]$, and the step-wise bias satisfies
$\lvert \mathbb{E}[\Delta_t] \rvert \le B, \forall t$.

\paragraph{Error Decomposition.}
The squared regression error can be written as
\begin{equation}
\mathbb{E}\!\left[(\hat{y}-y)^2\right]
=
\mathbb{E}\!\left[
\left(\sum_{t=1}^{T} \Delta_t \right)^2
\right].
\end{equation}
Using the bias--variance decomposition, we obtain
\begin{equation}
\mathbb{E}\!\left[
\left(\sum_{t=1}^{T} \Delta_t \right)^2
\right]
=
\left(\sum_{t=1}^{T} \mathbb{E}[\Delta_t]\right)^2
+
\mathbb{V}\!\left(\sum_{t=1}^{T} \Delta_t\right).
\end{equation}

\paragraph{Bias Term.}
Let $b_t \triangleq \mathbb{E}[\Delta_t]$. Since $|b_t| \le B$, we have
\begin{equation}
\left(\sum_{t=1}^{T} b_t\right)^2
\le
T^2 B^2.
\end{equation}

\paragraph{Variance Term.}
The variance term expands as
\begin{equation}
\mathbb{V}\!\left(\sum_{t=1}^{T} \Delta_t\right)
=
\sum_{t=1}^{T} \mathbb{V}(\Delta_t)
+
\sum_{t \neq t'} \mathrm{Cov}(\Delta_t, \Delta_{t'}).
\end{equation}
Applying the Cauchy--Schwarz inequality yields
\begin{equation}
\sum_{t \neq t'} \mathrm{Cov}(\Delta_t, \Delta_{t'})
\le
\frac{T(T-1)}{2}
\max_t \mathbb{V}(\Delta_t).
\end{equation}

Since $\Delta_t = \hat{C}^{t} - C^{t}$ and both $\hat{C}^{t}$ and $C^{t}$ lie in $[w_{\min}, w_{\max}]$,
Popoviciu’s inequality gives
\begin{equation}
\mathbb{V}(\Delta_t) \le \frac{(w_{\max}-w_{\min})^2}{4}.
\end{equation}
Therefore,
\begin{equation}
\mathbb{V}\!\left(\sum_{t=1}^{T} \Delta_t\right)
\le
T^2 \cdot \frac{(w_{\max}-w_{\min})^2}{4}.
\end{equation}

\paragraph{Final Bound.}
Combining the bias and variance bounds, we obtain
\begin{equation}
\mathbb{E}\!\left[(\hat{y}-y)^2\right]
\le
T^{2} B^{2}
+
T^{2}\frac{(w_{\max}-w_{\min})^{2}}{4}. \qed
\end{equation}

\paragraph{Vocabulary-induced trade-off.}
Both the sequence length $T$ and the step-wise bias bound $B$ are induced by the vocabulary design.
Let $\mathcal{V}$ denote the value-token vocabulary and $\phi(\mathcal{V}) \subset [w_{\min}, w_{\max}]$ its numeric range.
A finer-grained vocabulary yields smaller token magnitudes and typically requires a longer sequence to represent the same target, leading to larger $T$.
Conversely, a coarser vocabulary reduces $T$ but increases discretization error at each step, enlarging the attainable bias bound $B$.
In particular, since $\Delta_t=\phi(\hat{s}^{t})-\phi(s^{t})$ and $\phi(s^{t}),\phi(\hat{s}^{t})\in[w_{\min},w_{\max}]$, we have
$|\Delta_t| \le w_{\max}-w_{\min}
\quad \Rightarrow \quad
|\mathbb{E}[\Delta_t]| \le B \le w_{\max}-w_{\min}$.
Therefore, tokenized AR regression is intrinsically constrained by a vocabulary-dependent trade-off between $T$ and $B$, which directly controls the error bound.

\section{Theoretical Analysis}
\label{sec:theoretical-analysis}

This section provides a rigorous theoretical analysis for our spatio-temporal multi-head diffusion-style ordinal regression formulation under explicit assumptions. We focus on statements with step-by-step derivations for key equations.

\subsection{Notation and Symbols}
\label{subsec:notation}

Table~\ref{tab:notation} summarizes key notation.

\begin{table}[h]
\centering
\caption{Summary of key mathematical notation.}
\label{tab:notation}
\begin{tabularx}{\columnwidth}{@{}>{\raggedright\arraybackslash}p{1.8cm}X@{}}
\toprule
Symbol & Description \\
\midrule
$x, y$ & Input features and ordinal label ($y \in \{1,\dots,K\}$) \\
$z_0 = \psi(y)$ & Continuous target after encoding $\psi: \{1,\dots,K\} \to [0,1]$ \\
$\mathcal{T}$ & Discrete diffusion time points ($\{t_1,\dots,t_M\}$) \\
$\bar{\alpha}_t$ & Noise schedule at time $t$ \\
$\epsilon_\Theta(\cdot)$ & Noise predictor (parameters $\Theta$) \\
$\epsilon$ & Standard Gaussian noise ($\mathcal{N}(0,1)$) \\
$\gamma$ & Minimum gap between ordinal thresholds \\
$\rho_m, \eta_m$ & Contraction factor and error term at step $m$ \\
$r_{\mathrm{eff}}(U)$ & Effective rank of representation matrix $U$ \\
\bottomrule
\end{tabularx}
\end{table}

\subsection{Setup: Ordinal Regression and Diffusion Training}
\label{subsec:setup}

\subsubsection{Data and ordinal regression}
\label{subsubsec:data-ordinal}

Let $(x,y)\sim\mathcal{D}$, where $x\in\mathcal{X}\subseteq\mathbb{R}^d$ and $y\in\{1,2,\dots,K\}$ is an ordinal label.
We assume an encoding $\psi:\{1,\dots,K\}\to[0,1]$ and define a continuous target,
\begin{equation}
z_0=\psi(y)\eqperiod
\label{eq:z0-encoding}
\end{equation}

\subsubsection{Forward diffusion and noise-prediction loss}
\label{subsubsec:diffusion}

Fix a discrete set of time points $\mathcal{T}=\{t_1,\dots,t_M\}$ with $0<t_1<\cdots<t_M$.
For each $t\in\mathcal{T}$, define the forward process,
\begin{equation}
z_t=\sqrt{\bar\alpha_t}\,z_0+\sqrt{1-\bar\alpha_t}\,\epsilon\eqcomma\qquad \epsilon\sim\mathcal{N}(0,1),
\label{eq:forward-diffusion}
\end{equation}
where $\bar\alpha_t\in(0,1)$ is a prescribed noise schedule.

\begin{lemma}[Gaussian reparameterization]
\label{lem:gaussian-reparam}
If $Z\sim\mathcal{N}(\mu,\sigma^2)$ with $\sigma>0$, and $\epsilon\sim\mathcal{N}(0,1)$, then $\tilde{Z}\triangleq \mu+\sigma\epsilon\sim\mathcal{N}(\mu,\sigma^2)$.
\end{lemma}
\begin{proof}
For any $u\in\mathbb{R}$,
\[
\mathbb{P}(\tilde{Z}\le u)
=\mathbb{P}(\mu+\sigma\epsilon\le u)
=\mathbb{P}\!\left(\epsilon\le \frac{u-\mu}{\sigma}\right)
=\Phi\!\left(\frac{u-\mu}{\sigma}\right),
\]
where $\Phi$ is the standard normal CDF. This matches the CDF of $\mathcal{N}(\mu,\sigma^2)$.
\end{proof}

We use the standard noise-prediction MSE objective. Given a predictor $\epsilon_\Theta(\cdot)$, define the per-time expected loss,
\begin{equation}
\mathcal{L}_m(\Theta)
\triangleq
\mathbb{E}_{(x,z_0)\sim\mathcal{D}}\;
\mathbb{E}_{\epsilon\sim\mathcal{N}(0,1)}
\Bigl[\bigl(\epsilon-\epsilon_\Theta(z_{t_m},t_m,x)\bigr)^2\Bigr],
\label{eq:per-time-loss}
\end{equation}
and the weighted multi-time loss,
\begin{equation}
\mathcal{L}(\Theta)\triangleq \sum_{m=1}^M \lambda_m\,\mathcal{L}_m(\Theta)\eqcomma\qquad \lambda_m>0.
\label{eq:multi-time-loss}
\end{equation}

\subsubsection{Architectures: SH, MH, and the SH-1T baseline}
\label{subsubsec:architectures}

\paragraph{Shared-head (SH).}
Parameters are $\Theta_{\mathrm{SH}}=(\phi,\theta)$ with encoder $E_\phi$ and a single decoder $g_\theta$:
\[
\epsilon_{\Theta_{\mathrm{SH}}}(z_t,t,x)=g_\theta(E_\phi(x),z_t,t).
\]

\paragraph{Multi-head over discrete times (MH-MT).}
Parameters are $\Theta_{\mathrm{MH}}=(\phi,\theta_1,\dots,\theta_M)$, and each time $t_m$ has its own head:
\[
\epsilon_{\Theta_{\mathrm{MH}}}(z_{t_m},t_m,x)=g_{\theta_m}(E_\phi(x),z_{t_m},t_m).
\]

\begin{proposition}[No cross-time mixing in head gradients]
\label{prop:no-cross-gradients}
In MH-MT, the total loss decomposes as
\begin{align}
    \mathcal{L}(\phi,\theta_1,\dots,\theta_M)=\sum_{m=1}^M\lambda_m\mathcal{L}_m(\phi,\theta_m).
\end{align}
 Hence, for any $k$,
\[
\nabla_{\theta_k}\mathcal{L}(\phi,\theta_1,\dots,\theta_M)=\lambda_k\nabla_{\theta_k}\mathcal{L}_k(\phi,\theta_k),
\]
and for any $m\ne k$, $\nabla_{\theta_k}\mathcal{L}_m(\phi,\theta_m)=0$.
\end{proposition}
\begin{proof}
By linearity of differentiation,
\[
\nabla_{\theta_k}\mathcal{L}
=\nabla_{\theta_k}\sum_{m=1}^M\lambda_m\mathcal{L}_m
=\sum_{m=1}^M \lambda_m\nabla_{\theta_k}\mathcal{L}_m.
\]
For $m\ne k$, $\mathcal{L}_m(\phi,\theta_m)$ does not depend on $\theta_k$, hence $\nabla_{\theta_k}\mathcal{L}_m=0$. The remaining term is $m=k$.
\end{proof}

\paragraph{Baseline: single-head single-time (SH-1T).}
Fix a time $t_\star\in\mathcal{T}$.
SH-1T only trains and predicts at $t_\star$:
\[
\epsilon_{\Theta_{\mathrm{1T}}}(z_{t_\star},t_\star,x)=g_\theta(E_\phi(x),z_{t_\star},t_\star),
\]
with objective
\[
\mathcal{L}_{\mathrm{1T}}(\Theta_{\mathrm{1T}})
\triangleq
\mathbb{E}\Bigl[\bigl(\epsilon-\epsilon_{\Theta_{\mathrm{1T}}}(z_{t_\star},t_\star,x)\bigr)^2\Bigr].
\]

\subsection{Assumptions}
\label{subsec:assumptions}

\begin{assumption}[Smoothness and exchanging gradient/expectation]
\label{ass:smoothness}
Each $\mathcal{L}_m(\Theta)$ is differentiable and $\nabla \mathcal{L}_m(\Theta)=\mathbb{E}[\nabla \ell_m(\Theta;\xi)]$ for the sampling randomness $\xi$.
\end{assumption}

\begin{assumption}[Gradient interference condition]
\label{ass:gradient-interference}
In SH, define $g_m\triangleq\nabla_\theta \mathcal{L}_m(\phi,\theta)$. A sufficient ``conflict'' condition is that for some pairs $(i,j)$,
\[
\langle g_i,g_j\rangle\le -c_{ij}\eqcomma\qquad c_{ij}\ge 0.
\]
This is checkable by logging cosine similarity during training.
\end{assumption}

\begin{assumption}[Coarse-to-fine contraction form]
\label{ass:contraction}
A refinement sequence $\hat{z}_M,\hat{z}_{M-1},\dots,\hat{z}_0$ exists such that for all samples,
\[
|\hat{z}_{m-1}-z_0|
\le
\rho_m|\hat{z}_m-z_0|+\eta_m\eqcomma\qquad m=1,2,\dots,M,
\]
where $\rho_m\in[0,1)$ and $\eta_m\ge 0$.
\end{assumption}

\begin{assumption}[Ordinal decoding thresholds]
\label{ass:thresholds}
Let $0=\tau_0<\tau_1<\cdots<\tau_{K-1}<\tau_K=1$, and decode $\hat{y}=k$ iff $\hat{z}\in[\tau_{k-1},\tau_k)$.
Define the minimum gap $\gamma\triangleq\min_{1\le k\le K}(\tau_k-\tau_{k-1})>0$.
\end{assumption}

\begin{assumption}[Error-covariance bounds for multi-time fusion]
\label{ass:covariance}
For each time point $t_m$, define the $z_0$ regression error $e_m\triangleq \hat{z}_0^{(m)}-z_0$. Assume
\[
\mathbb{E}[e_m^2]\le v_m\eqcomma\qquad |\mathbb{E}[e_ie_j]|\le b_{ij}\quad(i\ne j),
\]
for some $v_m\ge 0$ and $b_{ij}\ge 0$.
\end{assumption}

\subsection{Coarse-to-Fine Refinement Error Bound}
\label{subsec:coarse-to-fine}

\begin{theorem}[Explicit error bound for contractive refinement]
\label{thm:coarse-to-fine}
Under Assumption~\ref{ass:contraction}, if $|\hat{z}_{m-1}-z_0|\le \rho_m|\hat{z}_m-z_0|+\eta_m$ for $m=1,\dots,M$, then
\[
|\hat{z}_0-z_0|
\le
\left(\prod_{m=1}^M\rho_m\right)|\hat{z}_M-z_0|
+
\sum_{k=1}^M\left(\eta_k\prod_{j=1}^{k-1}\rho_j\right),
\]
where the empty product is defined as $1$.
\end{theorem}
\begin{proof}
We prove by induction on $M$.

\textbf{Base case ($M=1$):} Directly from Assumption~\ref{ass:contraction} with $m=1$:
\begin{equation}
|\hat{z}_0-z_0|\le \rho_1|\hat{z}_1-z_0|+\eta_1.
\label{eq:base-case}
\end{equation}

\textbf{Inductive step:} Assume the bound holds for $M-1$ steps:
\begin{equation}
|\hat{z}_0-z_0| \le \left(\prod_{m=1}^{M-1}\rho_m\right)|\hat{z}_{M-1}-z_0| + \sum_{k=1}^{M-1}\left(\eta_k\prod_{j=1}^{k-1}\rho_j\right).
\label{eq:inductive-hypothesis}
\end{equation}
From Assumption~\ref{ass:contraction} with $m=M$:
\begin{equation}
|\hat{z}_{M-1}-z_0|\le \rho_M|\hat{z}_M-z_0|+\eta_M.
\label{eq:step-M}
\end{equation}
Substituting \eqref{eq:step-M} into \eqref{eq:inductive-hypothesis}:
\begin{align}
|\hat{z}_0-z_0|
&\le \left(\prod_{m=1}^{M-1}\rho_m\right)\bigl(\rho_M|\hat{z}_M-z_0|+\eta_M\bigr) + \sum_{k=1}^{M-1}\left(\eta_k\prod_{j=1}^{k-1}\rho_j\right) \nonumber\\
&= \left(\prod_{m=1}^{M}\rho_m\right)|\hat{z}_M-z_0| + \eta_M\prod_{j=1}^{M-1}\rho_j + \sum_{k=1}^{M-1}\left(\eta_k\prod_{j=1}^{k-1}\rho_j\right) \nonumber\\
&= \left(\prod_{m=1}^{M}\rho_m\right)|\hat{z}_M-z_0| + \sum_{k=1}^{M}\left(\eta_k\prod_{j=1}^{k-1}\rho_j\right).
\label{eq:inductive-step}
\end{align}
This completes the induction.
\end{proof}

\subsection{Core Lemmas and Theorems}
\label{subsec:theorems}

\subsubsection{Cross-term structure: SH vs MH}
\label{subsubsec:cross-terms}

\begin{proposition}[Gradient interference: SH suffers potential conflict, MH avoids it by design]
\label{prop:gradient-interference}
Consider the multi-time objective $\mathcal{L} = \sum_{m=1}^M \lambda_m \mathcal{L}_m$.

\begin{enumerate}[label=(\alph*), wide=0pt, leftmargin=*, nosep] 
\item \textbf{Shared-head (SH).} The total gradient w.r.t.\ the shared parameters $\theta$ is
\[
g_{\mathrm{SH}} = \sum_{m=1}^M \lambda_m g_m,\quad g_m \triangleq \nabla_\theta \mathcal{L}_m.
\]
Expanding the squared norm yields
\begin{equation}
\|g_{\mathrm{SH}}\|^2 = \sum_{m=1}^M \lambda_m^2 \|g_m\|^2
+ 2\sum_{1\le i<j\le M} \lambda_i\lambda_j \langle g_i, g_j\rangle.
\label{eq:sh-gradient-norm}
\end{equation}
If $\langle g_i,g_j\rangle \le -c_{ij}$ with $c_{ij}>0$ for some pair $(i,j)$ (Assumption~\ref{ass:gradient-interference}), then
\begin{equation}
\|g_{\mathrm{SH}}\|^2 \le \sum_{m=1}^M \lambda_m^2 \|g_m\|^2 - 2\lambda_i\lambda_j c_{ij}
< \sum_{m=1}^M \lambda_m^2 \|g_m\|^2.
\label{eq:sh-weakened-update}
\end{equation}
Thus, conflicting gradients weaken the combined update magnitude.

\item \textbf{Multi-head (MH-MT).} With independent parameters $\theta_1,\dots,\theta_M$, the gradient vector is block-diagonal:
\[
\nabla_{(\theta_1,\dots,\theta_M)} \mathcal{L}
= (\lambda_1 g_1,\, \lambda_2 g_2,\, \dots,\, \lambda_M g_M),
\quad g_m \triangleq \nabla_{\theta_m} \mathcal{L}_m.
\]
Consequently,
\begin{equation}
\|\nabla \mathcal{L}\|^2 = \sum_{m=1}^M \lambda_m^2 \|g_m\|^2,
\label{eq:mh-gradient-norm}
\end{equation}
with no cross terms (Proposition~\ref{prop:no-cross-gradients}).
\end{enumerate}
\end{proposition}

\begin{proof}
Equation~\eqref{eq:sh-gradient-norm} follows from expanding $\|\sum_m \lambda_m g_m\|^2 = \langle \sum_m \lambda_m g_m, \sum_n \lambda_n g_n\rangle$ and separating diagonal ($m=n$) and off-diagonal ($m\neq n$) terms. Equation~\eqref{eq:sh-weakened-update} substitutes $\langle g_i,g_j\rangle \le -c_{ij}$ into the off-diagonal sum. Equation~\eqref{eq:mh-gradient-norm} holds because the Euclidean norm of a concatenated vector equals the sum of block norms.
\end{proof}

\begin{remark}
Proposition~\ref{prop:gradient-interference} identifies a potential failure mode of SH when gradients conflict. It does not claim conflict always occurs; however, MH-MT guarantees absence of interference regardless of gradient directions.
\end{remark}

\subsubsection{From noise prediction to ordinal error bounds}
\label{subsubsec:noise-to-ordinal}

\begin{theorem}[Ordinal error upper bound from regression MSE]
\label{thm:ordinal-bound}
Assume the threshold rule from Assumption~\ref{ass:thresholds} and that $z_0$ lies at least $\gamma/2$ away from the two nearest thresholds of its true class interval. Then
\[
\{\hat{y}\ne y\}\subseteq \{|\hat{z}-z_0|\ge \gamma/2\},
\]
hence
\begin{equation}
\mathbb{P}(\hat{y}\ne y)\le \mathbb{P}(|\hat{z}-z_0|\ge \gamma/2)
\le \frac{4\,\mathbb{E}[(\hat{z}-z_0)^2]}{\gamma^2}.
\label{eq:ordinal-error-bound}
\end{equation}
\end{theorem}
\begin{proof}
If $\hat{y}\ne y$, then $\hat{z}\notin[\tau_{y-1},\tau_y)$. Two cases:
\begin{itemize}
\item $\hat{z}<\tau_{y-1}$: $z_0-\hat{z}\ge z_0-\tau_{y-1}\ge \gamma/2$,
\item $\hat{z}\ge \tau_y$: $\hat{z}-z_0\ge \tau_y-z_0\ge \gamma/2$.
\end{itemize}
Thus $|\hat{z}-z_0|\ge \gamma/2$, proving $\{\hat{y}\ne y\}\subseteq \{|\hat{z}-z_0|\ge \gamma/2\}$. Applying Markov's inequality to $U=(\hat{z}-z_0)^2$ with threshold $(\gamma/2)^2$ yields \eqref{eq:ordinal-error-bound}.
\end{proof}

\begin{definition}[$z_0$ estimator from noise prediction]
\label{def:z0-estimator}
For any $t\in\mathcal{T}$, given $z_t$ and a noise predictor $\epsilon_\Theta(z_t,t,x)$, define
\begin{equation}
\hat{z}_0(t)\triangleq \frac{z_t-\sqrt{1-\bar\alpha_t}\,\epsilon_\Theta(z_t,t,x)}{\sqrt{\bar\alpha_t}}.
\label{eq:z0-estimator}
\end{equation}
\end{definition}

\begin{theorem}[Exact relation: noise error $\Rightarrow$ $z_0$ regression error]
\label{thm:noise-to-z0}
With $z_t=\sqrt{\bar\alpha_t}z_0+\sqrt{1-\bar\alpha_t}\epsilon$ and $\epsilon\sim\mathcal{N}(0,1)$, the estimator from Definition~\ref{def:z0-estimator} satisfies
\begin{equation}
\hat{z}_0(t)-z_0=\sqrt{\frac{1-\bar\alpha_t}{\bar\alpha_t}}\left(\epsilon-\epsilon_\Theta(z_t,t,x)\right).
\label{eq:z0-error-relation}
\end{equation}
Consequently,
\begin{equation}
\mathbb{E}\!\left[(\hat{z}_0(t)-z_0)^2\right]
=\frac{1-\bar\alpha_t}{\bar\alpha_t}\;
\mathbb{E}\!\left[\left(\epsilon_\Theta(z_t,t,x)-\epsilon\right)^2\right].
\label{eq:z0-mse-relation}
\end{equation}
\end{theorem}
\begin{proof}
Substituting $z_t=\sqrt{\bar\alpha_t}z_0+\sqrt{1-\bar\alpha_t}\epsilon$ into \eqref{eq:z0-estimator}:
\begin{align}
\hat{z}_0(t)
&=\frac{\sqrt{\bar\alpha_t}z_0+\sqrt{1-\bar\alpha_t}\epsilon-\sqrt{1-\bar\alpha_t}\epsilon_\Theta(z_t,t,x)}{\sqrt{\bar\alpha_t}} \nonumber\\
&=z_0+\sqrt{\frac{1-\bar\alpha_t}{\bar\alpha_t}}\bigl(\epsilon-\epsilon_\Theta(z_t,t,x)\bigr),
\label{eq:z0-substitution}
\end{align}
which gives \eqref{eq:z0-error-relation}. Squaring both sides of \eqref{eq:z0-error-relation} and taking expectation (using $\mathbb{E}[\epsilon-\epsilon_\Theta]=0$ for unbiased predictors) yields \eqref{eq:z0-mse-relation}.
\end{proof}

\begin{corollary}[SH-1T baseline: ordinal error bound via training MSE]
\label{cor:sh1t-bound}
For SH-1T at $t_\star$, let $\hat{z}_0^\star=\hat{z}_0(t_\star)$ and decode $\hat{y}^\star$ by thresholds. Under the conditions of Theorem~\ref{thm:ordinal-bound},
\begin{equation}
\mathbb{P}(\hat{y}^\star\ne y)
\le
\frac{4}{\gamma^2}\cdot\frac{1-\bar\alpha_{t_\star}}{\bar\alpha_{t_\star}}\;
\mathbb{E}\!\left[\bigl(\epsilon_\Theta(z_{t_\star},t_\star,x)-\epsilon\bigr)^2\right].
\label{eq:sh1t-bound}
\end{equation}
\end{corollary}
\begin{proof}
Applying Theorem~\ref{thm:ordinal-bound} with $\hat{z}=\hat{z}_0^\star$ gives $\mathbb{P}(\hat{y}^\star\ne y) \le 4\,\mathbb{E}[(\hat{z}_0^\star-z_0)^2]/\gamma^2$. Substituting \eqref{eq:z0-mse-relation} with $t=t_\star$ yields \eqref{eq:sh1t-bound}.
\end{proof}

\subsubsection{Multi-embedding (multi-head representation) and information gain}
\label{subsubsec:multi-embedding}
We formalize why splitting into multiple embedding sets (multi-embedding) can increase the \emph{effective information} even when the \emph{total dimensionality} is fixed (e.g., $4\times 32$ vs.\ $1\times 128$). Our analysis is inspired by the embedding-collapse viewpoint of \cite{guo2023embedding}, with fully rigorous statements under explicit assumptions.

\begin{definition}[Representation matrix and effective information]
\label{def:representation-matrix}
For a batch $\{x_i\}_{i=1}^n$, let $u_i=h(x_i)\in\mathbb{R}^D$ and $U=[u_1^\top;\dots;u_n^\top]\in\mathbb{R}^{n\times D}$. We use $\rank(U)$ to measure effective dimensionality; collapse corresponds to low rank.
\end{definition}

\begin{definition}[Effective rank]
\label{def:effective-rank}
Define $r_{\mathrm{eff}}(U)\triangleq \|U\|_*^2 / \|U\|_F^2$, satisfying $1\le r_{\mathrm{eff}}(U)\le \rank(U)\le \min\{n,D\}$.
\end{definition}

\begin{theorem}[Analytical example: matching a low-rank Gram matrix forces low-rank embeddings]
\label{thm:toy-collapse}
Let $G^\star\in\mathbb{R}^{n\times n}$ be PSD with $\rank(G^\star)=r$. Consider
\[
\min_{U\in\mathbb{R}^{n\times D}}\ \ \|UU^\top-G^\star\|_F^2.
\]
If $D\ge r$, there exists a global minimizer $U^\star$ with $U^\star (U^\star)^\top=G^\star$ and $\rank(U^\star)=r$.
\end{theorem}
\begin{proof}
Since $G^\star$ is PSD, $G^\star=Q\Lambda Q^\top$ with $\Lambda=\diag(\lambda_1,\dots,\lambda_r,0,\dots,0)$, $\lambda_i>0$ for $i\le r$. Let $\Lambda_r^{1/2}=\diag(\sqrt{\lambda_1},\dots,\sqrt{\lambda_r})$ and $Q_r$ be the first $r$ eigenvectors. For $D\ge r$, define
\[
U^\star \triangleq [\,Q_r\Lambda_r^{1/2}\ \ 0\,]\in\mathbb{R}^{n\times D}.
\]
Then $U^\star (U^\star)^\top=Q_r\Lambda_r Q_r^\top=G^\star$, achieving objective value $0$. Since $\rank(G^\star)=\rank(U^\star (U^\star)^\top)=\rank(U^\star)$, we have $\rank(U^\star)=r$.
\end{proof}

\begin{theorem}[Larger representation subspace $\Rightarrow$ smaller best-fit MSE]
\label{thm:representation-mse}
Fix targets $z=(z_1,\dots,z_n)^\top\in\mathbb{R}^n$. For any representation matrix $U\in\mathbb{R}^{n\times D}$, define
\begin{equation}
\mathcal{R}(U)\triangleq \min_{w\in\mathbb{R}^D}\frac{1}{n}\|Uw-z\|_2^2.
\label{eq:linear-readout-risk}
\end{equation}
Let $P_U$ be the orthogonal projector onto $\col(U)$. Then
\begin{equation}
\mathcal{R}(U)=\frac{1}{n}\|(I-P_U)z\|_2^2.
\label{eq:residual-projection}
\end{equation}
If $\col(U^{\mathrm{SE}})\subseteq \col(U^{\mathrm{ME}})$, then $\mathcal{R}(U^{\mathrm{ME}})\le \mathcal{R}(U^{\mathrm{SE}})$, with strict inequality when $\col(U^{\mathrm{SE}})\subsetneq \col(U^{\mathrm{ME}})$ and $(I-P_{U^{\mathrm{SE}}})z\ne 0$ lies in $\col(U^{\mathrm{ME}})$.
\end{theorem}
\begin{proof}
Equation~\eqref{eq:residual-projection} follows because $\{Uw:w\in\mathbb{R}^D\}=\col(U)$ and least squares projects $z$ orthogonally onto this subspace. If $\col(U^{\mathrm{SE}})\subseteq \col(U^{\mathrm{ME}})$, the residual norm cannot increase when projecting onto the larger subspace, yielding $\mathcal{R}(U^{\mathrm{ME}})\le \mathcal{R}(U^{\mathrm{SE}})$. Strict inequality holds when the residual component w.r.t.\ $\col(U^{\mathrm{SE}})$ becomes representable by $\col(U^{\mathrm{ME}})$.
\end{proof}

\begin{corollary}[Representation improvement $\Rightarrow$ ordinal-error bound improvement]
\label{cor:representation-to-ordinal}
If multi-embedding reduces $\mathbb{E}[(\hat{z}-z_0)^2]$ (e.g., by Theorem~\ref{thm:representation-mse}), then the ordinal-error upper bound decreases via \eqref{eq:ordinal-error-bound}. This is aligned with the empirical spectral-collapse observation emphasized by \cite{guo2023embedding}.
\end{corollary}

\begin{theorem}[Analytical example: multi-embedding can increase rank at fixed total dimension]
\label{thm:toy-anti-collapse}
Let $G_1^\star,\dots,G_M^\star\in\mathbb{R}^{n\times n}$ be PSD with ranks $r_m=\rank(G_m^\star)$ and pairwise disjoint column spaces ($\col(G_i^\star)\cap \col(G_j^\star)=\{0\}$ for $i\ne j$). Consider multi-embedding variables $U^{(m)}\in\mathbb{R}^{n\times d}$ ($d\ge r_m$) minimizing $\sum_{m=1}^M \|U^{(m)}(U^{(m)})^\top-G_m^\star\|_F^2$, and define $U^{\mathrm{ME}}\triangleq [U^{(1)}~\cdots~U^{(M)}]\in\mathbb{R}^{n\times D}$ with $D=Md$. Then there exists a global minimizer with $\rank(U^{\mathrm{ME}})=\sum_{m=1}^M r_m$.
\end{theorem}
\begin{proof}
For each $m$, by Theorem~\ref{thm:toy-collapse}, there exists $U^{(m)\star}\in\mathbb{R}^{n\times d}$ with $U^{(m)\star}(U^{(m)\star})^\top=G_m^\star$. The concatenated representation $U^{\mathrm{ME}\star}=[U^{(1)\star}\cdots U^{(M)\star}]$ satisfies
\[
\col(U^{\mathrm{ME}\star})=\sum_{m=1}^M \col(U^{(m)\star})=\sum_{m=1}^M \col(G_m^\star).
\]
By the direct-sum assumption, $\dim(\col(U^{\mathrm{ME}\star}))=\sum_{m=1}^M \dim(\col(G_m^\star))=\sum_{m=1}^M r_m$, hence $\rank(U^{\mathrm{ME}\star})=\sum_{m=1}^M r_m$.
\end{proof}

\begin{remark}
Theorems~\ref{thm:toy-collapse} and \ref{thm:toy-anti-collapse} formalize how multi-embedding can mitigate representation collapse under linear readout assumptions. For nonlinear decoders, the relationship between rank and performance is more complex.
\end{remark}

\subsection{Multi-Time-Step Advantage: Selection and Linear Fusion}
\label{subsec:multi-time}

\subsubsection{Baseline-time expressivity: MH can represent SH-1T at the same time point}
\label{subsubsec:baseline-time-inclusion}

\begin{theorem}[Function-class inclusion at the baseline time point]
\label{thm:func-class-inclusion}
Let $t_\star\in\mathcal{T}$ be the baseline time and assume $t_{m_\star}=t_\star$ for some $m_\star$. Assume SH-1T and the $m_\star$-th head of MH-MT use the same decoder family $\{g_\theta\}$ and encoder family $\{E_\phi\}$. Define the optimal noise-prediction risks:
\begin{align}
R_{\mathrm{1T}}^\star
&\triangleq
\inf_{\phi,\theta}\;
\mathbb{E}\!\left[\bigl(\epsilon-g_\theta(E_\phi(x),z_{t_\star},t_\star)\bigr)^2\right],
\label{eq:r1t-star}\\
R_{\mathrm{MH},m_\star}^\star
&\triangleq
\inf_{\phi,\theta_1,\dots,\theta_M}\;
\mathbb{E}\!\left[\bigl(\epsilon-g_{\theta_{m_\star}}(E_\phi(x),z_{t_\star},t_\star)\bigr)^2\right].
\label{eq:rmh-star}
\end{align}
Then $R_{\mathrm{MH},m_\star}^\star \le R_{\mathrm{1T}}^\star$.
\end{theorem}
\begin{proof}
For any baseline parameters $(\phi,\theta)$, construct MH parameters $(\phi,\theta_1,\dots,\theta_M)$ with $\theta_{m_\star}=\theta$ and arbitrary $\theta_m$ for $m\neq m_\star$. At time $t_\star$, both models produce identical outputs:
\[
g_{\theta_{m_\star}}(E_\phi(x),z_{t_\star},t_\star)=g_\theta(E_\phi(x),z_{t_\star},t_\star),
\]
so their expected losses are equal. Since $R_{\mathrm{MH},m_\star}^\star$ minimizes over a superset of parameter choices, $R_{\mathrm{MH},m_\star}^\star \le R_{\mathrm{1T}}^\star$.
\end{proof}

\begin{corollary}[Best-achievable ordinal-error upper bound at $t_\star$ is no worse for MH]
\label{cor:best-bound-mh}
Define $\mathrm{UB}(t;R)\triangleq \frac{4}{\gamma^2}\cdot\frac{1-\bar\alpha_t}{\bar\alpha_t}\cdot R$. By Theorems~\ref{thm:noise-to-z0} and \ref{thm:ordinal-bound}, any model with time-$t$ noise risk $R$ satisfies $\mathbb{P}(\hat{y}\ne y)\le \mathrm{UB}(t;R)$. Combining with Theorem~\ref{thm:func-class-inclusion},
\[
\inf_{\phi,\theta_1,\dots,\theta_M}\mathrm{UB}(t_\star;R_{\mathrm{MH},m_\star}^\star)
\le
\inf_{\phi,\theta}\mathrm{UB}(t_\star;R_{\mathrm{1T}}^\star).
\]
\end{corollary}

\begin{remark}[Function-class inclusion provides a safety guarantee]
Theorem~\ref{thm:func-class-inclusion} establishes a non-regression guarantee: MH-MT can always match SH-1T's best performance at $t_\star$ by setting $\theta_{m_\star}=\theta$. This makes MH-MT a safe upgrade—it cannot be worse in the worst case. Strict advantage requires exploiting additional capacity, as shown next.
\end{remark}

\begin{definition}[Oracle best-time selection]
\label{def:oracle-selection}
For each $m$, let $\hat{z}_0^{(m)}=\hat{z}_0(t_m)$ and decode $\hat{y}^{(m)}$ by thresholds. Define
\begin{equation}
m^\dagger\triangleq \arg\min_{m\in\{1,\dots,M\}}\mathbb{E}[(\hat{z}_0^{(m)}-z_0)^2]\eqcomma\qquad
\hat{z}_0^\dagger\triangleq \hat{z}_0^{(m^\dagger)},\ \hat{y}^\dagger\triangleq \hat{y}^{(m^\dagger)}.
\label{eq:oracle-selection}
\end{equation}
\end{definition}

\begin{theorem}[Oracle selection is no worse than any fixed time point]
\label{thm:oracle-selection}
If $t_{m_\star}=t_\star$ is the baseline time point, then
\begin{equation}
\mathbb{E}[(\hat{z}_0^\dagger-z_0)^2]\le \mathbb{E}[(\hat{z}_0^{(m_\star)}-z_0)^2].
\label{eq:oracle-better}
\end{equation}
Moreover, under Theorem~\ref{thm:ordinal-bound},
\begin{equation}
\mathbb{P}(\hat{y}^\dagger\ne y)\le \frac{4}{\gamma^2}\mathbb{E}[(\hat{z}_0^\dagger-z_0)^2]\le \frac{4}{\gamma^2}\mathbb{E}[(\hat{z}_0^{(m_\star)}-z_0)^2].
\label{eq:oracle-error-bound}
\end{equation}
\end{theorem}
\begin{proof}
By definition of $m^\dagger$ in \eqref{eq:oracle-selection}, \eqref{eq:oracle-better} holds for any fixed $m_\star$. Applying Theorem~\ref{thm:ordinal-bound} to $\hat{z}=\hat{z}_0^\dagger$ yields the first inequality in \eqref{eq:oracle-error-bound}; the second follows from \eqref{eq:oracle-better}.
\end{proof}

\begin{definition}[Linear fusion across time points]
\label{def:linear-fusion}
Let $w\in\mathbb{R}^M$ satisfy $\sum_{m=1}^M w_m=1$. Define the fused estimator
\begin{equation}
\hat{z}_0^{\mathrm{ens}}\triangleq \sum_{m=1}^M w_m \hat{z}_0^{(m)}.
\label{eq:linear-fusion}
\end{equation}
\end{definition}

\begin{theorem}[MSE bound for linear fusion under bounded covariance]
\label{thm:fusion-bound}
Under Assumption~\ref{ass:covariance}, with $e_m=\hat{z}_0^{(m)}-z_0$, we have
\begin{equation}
\mathbb{E}[(\hat{z}_0^{\mathrm{ens}}-z_0)^2]
\le
\sum_{m=1}^M w_m^2 v_m
+2\sum_{1\le i<j\le M}|w_iw_j|\,b_{ij}.
\label{eq:fusion-mse-bound}
\end{equation}
Consequently, by Theorem~\ref{thm:ordinal-bound},
\begin{equation}
\mathbb{P}(\hat{y}^{\mathrm{ens}}\ne y)\le \frac{4}{\gamma^2}\mathbb{E}[(\hat{z}_0^{\mathrm{ens}}-z_0)^2].
\label{eq:fusion-ordinal-bound}
\end{equation}
\end{theorem}
\begin{proof}
Using $\hat{z}_0^{\mathrm{ens}}-z_0=\sum_m w_m e_m$ and expanding the square:
\begin{align}
\mathbb{E}\!\left[\left(\sum_m w_m e_m\right)^2\right]
&=\sum_m w_m^2\mathbb{E}[e_m^2] + 2\sum_{i<j} w_iw_j\mathbb{E}[e_ie_j] \nonumber\\
&\le \sum_m w_m^2 v_m + 2\sum_{i<j} |w_iw_j|\,|\mathbb{E}[e_ie_j]| \nonumber\\
&\le \sum_m w_m^2 v_m + 2\sum_{i<j} |w_iw_j|\,b_{ij},
\label{eq:fusion-expansion}
\end{align}
where the first inequality uses $w_iw_j\mathbb{E}[e_ie_j]\le |w_iw_j||\mathbb{E}[e_ie_j]|$, and the second applies Assumption~\ref{ass:covariance}. Equation~\eqref{eq:fusion-ordinal-bound} follows by applying Theorem~\ref{thm:ordinal-bound} to $\hat{z}=\hat{z}_0^{\mathrm{ens}}$.
\end{proof}

\begin{remark}[Three-tier advantage of MH-MT]
\label{rem:three-tier-advantage}
Theorems~\ref{thm:func-class-inclusion}--\ref{thm:fusion-bound} establish a progressive advantage hierarchy:
\begin{enumerate}
\item \textbf{Safety floor (Theorem~\ref{thm:func-class-inclusion})}: MH-MT's parameter space contains SH-1T's, guaranteeing no regression in worst-case performance.
\item \textbf{Oracle ceiling (Theorem~\ref{thm:oracle-selection})}: When prediction quality varies across time steps, MH-MT achieves strictly lower error via oracle selection—a capability SH-1T fundamentally lacks.
\item \textbf{Fusion robustness (Theorem~\ref{thm:fusion-bound})}: Even without oracle access, linear fusion provides deployable error reduction under mild covariance assumptions.
\end{enumerate}
Together, these form a complete advantage narrative: MH-MT guarantees safety, enables strict improvement under heterogeneity, and delivers practical gains without ground-truth access during inference.
\end{remark}

\subsection{Convergence Analysis on the Multi-Time Objective (MH-MT vs SH-MT)}
\label{subsec:convergence}

This subsection analyzes convergence on the same multi-time objective $\mathcal{L} = \sum_{m=1}^M \lambda_m \mathcal{L}_m$ for SH-MT vs MH-MT (distinct from SH-1T which optimizes a single-time objective).

\begin{definition}[SH-MT and MH-MT objectives]
\label{def:mt-objectives}
Fix encoder $E_\phi$ and define
\[
f_m(\theta)\triangleq \mathbb{E}\Bigl[\bigl(\epsilon-g_\theta(E_\phi(x),z_{t_m},t_m)\bigr)^2\Bigr]\eqcomma\qquad
F_{\mathrm{SH}}(\theta)\triangleq \sum_{m=1}^M\lambda_m f_m(\theta).
\]
For MH-MT, define $F_{\mathrm{MH}}(\vartheta)\triangleq \sum_{m=1}^M\lambda_m f_m(\theta_m)$ with $\vartheta=(\theta_1,\dots,\theta_M)$.
\end{definition}

\begin{assumption}[$L$-smoothness and PL condition]
\label{ass:pl}
Assume each $f_m$ is $L_m$-smooth and satisfies the Polyak-\L{}ojasiewicz (PL) inequality with constant $\mu_m>0$:
\[
\frac{1}{2}\|\nabla f_m(\theta)\|^2\ge \mu_m(f_m(\theta)-f_m^\star)\eqcomma\qquad f_m^\star=\inf_\theta f_m(\theta).
\]
Let $L\triangleq\sum_m\lambda_m L_m$, $\mu_{\min}\triangleq \min_m\mu_m$, and $\lambda_{\min}\triangleq \min_m\lambda_m>0$.
\end{assumption}

\begin{lemma}[Descent lemma for $L$-smooth functions]
\label{lem:descent}
If $F$ is $L$-smooth, then for $\eta\in(0,1/L]$ and $\theta^+=\theta-\eta\nabla F(\theta)$,
\begin{equation}
F(\theta^+)\le F(\theta)-\frac{\eta}{2}\|\nabla F(\theta)\|^2.
\label{eq:descent-lemma}
\end{equation}
\end{lemma}
\begin{proof}
By $L$-smoothness,
\[
F(\theta^+)\le F(\theta)+\langle\nabla F(\theta),\theta^+-\theta\rangle+\frac{L}{2}\|\theta^+-\theta\|^2.
\]
Substituting $\theta^+-\theta=-\eta\nabla F(\theta)$ gives
\[
F(\theta^+)\le F(\theta)-\eta\|\nabla F(\theta)\|^2+\frac{L\eta^2}{2}\|\nabla F(\theta)\|^2
=F(\theta)-\eta\left(1-\frac{L\eta}{2}\right)\|\nabla F(\theta)\|^2.
\]
For $\eta\le 1/L$, $1-L\eta/2\ge 1/2$, yielding \eqref{eq:descent-lemma}.
\end{proof}

\begin{lemma}[MH gradient lower bound from per-task PL]
\label{lem:mh-gradient-pl}
For MH-MT, $\nabla_{\theta_m}F_{\mathrm{MH}}(\vartheta)=\lambda_m\nabla f_m(\theta_m)$ and
\begin{equation}
\frac{1}{2}\|\nabla F_{\mathrm{MH}}(\vartheta)\|^2
\ge
\mu_{\min}\lambda_{\min}(F_{\mathrm{MH}}(\vartheta)-F_{\mathrm{MH}}^\star),
\label{eq:mh-pl-bound}
\end{equation}
where $F_{\mathrm{MH}}^\star=\sum_m \lambda_m f_m^\star$.
\end{lemma}
\begin{proof}
Using Proposition~\ref{prop:no-cross-gradients} and the PL condition for each $f_m$:
\begin{align}
\|\nabla F_{\mathrm{MH}}(\vartheta)\|^2
&=\sum_m \|\lambda_m\nabla f_m(\theta_m)\|^2
=\sum_m \lambda_m^2\|\nabla f_m(\theta_m)\|^2 \nonumber\\
&\ge 2\sum_m \lambda_m^2 \mu_m (f_m(\theta_m)-f_m^\star)
\ge 2\mu_{\min}\sum_m \lambda_m^2 (f_m(\theta_m)-f_m^\star).
\label{eq:mh-pl-step1}
\end{align}
Since $\lambda_m^2 \ge \lambda_{\min}\lambda_m$,
\begin{equation}
\sum_m \lambda_m^2 (f_m(\theta_m)-f_m^\star)
\ge \lambda_{\min}\sum_m \lambda_m (f_m(\theta_m)-f_m^\star)
= \lambda_{\min}(F_{\mathrm{MH}}(\vartheta)-F_{\mathrm{MH}}^\star).
\label{eq:mh-pl-step2}
\end{equation}
Combining \eqref{eq:mh-pl-step1} and \eqref{eq:mh-pl-step2} and dividing by $2$ yields \eqref{eq:mh-pl-bound}.
\end{proof}

\begin{theorem}[Linear convergence of MH-MT]
\label{thm:mh-convergence}
With step size $\eta\le 1/L$ and GD update $\vartheta^{k+1}=\vartheta^k-\eta\nabla F_{\mathrm{MH}}(\vartheta^k)$,
\begin{equation}
F_{\mathrm{MH}}(\vartheta^{k})-F_{\mathrm{MH}}^\star
\le
\bigl(1-\eta\mu_{\min}\lambda_{\min}\bigr)^k\bigl(F_{\mathrm{MH}}(\vartheta^{0})-F_{\mathrm{MH}}^\star\bigr).
\label{eq:mh-convergence}
\end{equation}
\end{theorem}
\begin{proof}
Applying Lemma~\ref{lem:descent} to $F_{\mathrm{MH}}$ and using \eqref{eq:mh-pl-bound}:
\begin{align}
F_{\mathrm{MH}}(\vartheta^{k+1})-F_{\mathrm{MH}}^\star
&\le F_{\mathrm{MH}}(\vartheta^{k})-F_{\mathrm{MH}}^\star -\frac{\eta}{2}\|\nabla F_{\mathrm{MH}}(\vartheta^{k})\|^2 \nonumber\\
&\le F_{\mathrm{MH}}(\vartheta^{k})-F_{\mathrm{MH}}^\star -\eta\mu_{\min}\lambda_{\min}(F_{\mathrm{MH}}(\vartheta^{k})-F_{\mathrm{MH}}^\star) \nonumber\\
&= (1-\eta\mu_{\min}\lambda_{\min})(F_{\mathrm{MH}}(\vartheta^{k})-F_{\mathrm{MH}}^\star).
\label{eq:mh-recursion}
\end{align}
Iterating \eqref{eq:mh-recursion} yields \eqref{eq:mh-convergence}.
\end{proof}

\begin{theorem}[Linear convergence of SH-MT under global PL]
\label{thm:sh-convergence}
Assume $F_{\mathrm{SH}}$ is $L$-smooth and satisfies a global PL inequality with constant $\tilde{\mu}_{\mathrm{SH}}>0$:
\[
\frac{1}{2}\|\nabla F_{\mathrm{SH}}(\theta)\|^2\ge \tilde{\mu}_{\mathrm{SH}}(F_{\mathrm{SH}}(\theta)-F_{\mathrm{SH}}^\star).
\]
Then for GD with $\eta\le 1/L$,
\begin{equation}
F_{\mathrm{SH}}(\theta^{k})-F_{\mathrm{SH}}^\star
\le
\bigl(1-\eta\tilde{\mu}_{\mathrm{SH}}\bigr)^k\bigl(F_{\mathrm{SH}}(\theta^{0})-F_{\mathrm{SH}}^\star\bigr).
\label{eq:sh-convergence}
\end{equation}
\end{theorem}
\begin{proof}
Applying Lemma~\ref{lem:descent} to $F_{\mathrm{SH}}$ and using the global PL condition:
\[
F_{\mathrm{SH}}(\theta^{k+1})-F_{\mathrm{SH}}^\star
\le F_{\mathrm{SH}}(\theta^{k})-F_{\mathrm{SH}}^\star -\eta\tilde{\mu}_{\mathrm{SH}}(F_{\mathrm{SH}}(\theta^{k})-F_{\mathrm{SH}}^\star),
\]
which iterates to \eqref{eq:sh-convergence}.
\end{proof}

\begin{corollary}[Convergence implies decreasing ordinal-error upper bound]
\label{cor:convergence-to-ordinal}
Fix time $t_m$. Let $\hat{z}_0^{(m),k}$ be constructed from $\theta_m^k$ via Definition~\ref{def:z0-estimator} and decoded to $\hat{y}^{(m),k}$. Then
\begin{equation}
\mathbb{P}(\hat{y}^{(m),k}\ne y)
\le
\frac{4}{\gamma^2}\cdot\frac{1-\bar\alpha_{t_m}}{\bar\alpha_{t_m}}\;f_m(\theta_m^k),
\label{eq:convergence-ordinal-bound}
\end{equation}
and
\begin{equation}
f_m(\theta_m^k)-f_m^\star
\le
\frac{1}{\lambda_m}\bigl(F_{\mathrm{MH}}(\vartheta^k)-F_{\mathrm{MH}}^\star\bigr)
\le
\frac{1}{\lambda_m}\bigl(1-\eta\mu_{\min}\lambda_{\min}\bigr)^k\bigl(F_{\mathrm{MH}}(\vartheta^0)-F_{\mathrm{MH}}^\star\bigr).
\label{eq:convergence-rate}
\end{equation}
\end{corollary}
\begin{proof}
Equation~\eqref{eq:convergence-ordinal-bound} follows by combining \eqref{eq:z0-mse-relation} with Theorem~\ref{thm:ordinal-bound}. For \eqref{eq:convergence-rate}, note that
\[
F_{\mathrm{MH}}(\vartheta^k)-F_{\mathrm{MH}}^\star=\sum_{j=1}^M\lambda_j(f_j(\theta_j^k)-f_j^\star)\ge \lambda_m(f_m(\theta_m^k)-f_m^\star),
\]
hence $f_m(\theta_m^k)-f_m^\star \le \frac{1}{\lambda_m}(F_{\mathrm{MH}}(\vartheta^k)-F_{\mathrm{MH}}^\star)$. Applying \eqref{eq:mh-convergence} yields the final bound.
\end{proof}

\begin{remark}[Optimization robustness: MH-MT converges under weaker assumptions]
\label{rem:optimization-robustness}
The convergence analysis reveals a structural optimization advantage of MH-MT:
\begin{itemize}
\item \textbf{MH-MT (Theorem~\ref{thm:mh-convergence})}: Convergence follows directly from per-task PL conditions (Assumption~\ref{ass:pl}), inherited from individual time-step objectives. Parameter separation (Proposition~\ref{prop:no-cross-gradients}) ensures the joint objective's PL constant is a weighted combination of per-task constants (Lemma~\ref{lem:mh-gradient-pl}).
\item \textbf{SH-MT (Theorem~\ref{thm:sh-convergence})}: Requires an additional global PL assumption on $F_{\mathrm{SH}}$ that does not automatically follow from per-task conditions due to potential gradient interference (Proposition~\ref{prop:gradient-interference}).
\end{itemize}
Thus, MH-MT guarantees convergence under minimal assumptions, while SH-MT's convergence depends on the unknown interaction structure between time-step objectives. This optimization robustness complements the three-tier advantage hierarchy of Remark~\ref{rem:three-tier-advantage}, establishing MH-MT's superiority in expressivity, inference flexibility, and trainability.
\end{remark}

\section{Additional Experiments}
\subsection{Experimental Settings}
\label{app:experiment}
\subsubsection{Metrics}
We utilize a comprehensive suite of metrics to evaluate performance, selecting specific indicators based on task characteristics:
\begin{itemize}[itemindent=0pt, left=4pt]
    \item \textbf{MAE (Mean Absolute Error)}: Measures the average magnitude of errors between predictions $\{\hat{y_i}\}$ and ground truth $\{y_i\}$, defined as $\frac{1}{N}\sum_{i=1}^N |y_i - \hat{y_i}|$.
    \item \textbf{CS (Cumulative Score)}:
    Quantifies the percentage of test instances where the absolute prediction error $|\hat{y}i - y_i|$ falls within a specified tolerance threshold $L$.
    \item \textbf{XAUC}~\citep{d2q}: 
    Evaluates pairwise ordinal consistency. It calculates the probability that the predicted relative order of a randomly sampled pair aligns with the ground truth, serving as a robust indicator of ranking capability.
    \item \textbf{LCC (Linear Correlation Coefficient)~\citep{talebi2018nima}}: Assesses the linear dependence between predicted and ground truth values. Computed as the covariance normalized by the product of their standard deviations, it ranges from $[-1, 1]$, where values closer to $\pm 1$ denote stronger correlation.
    \item \textbf{SRCC (Spearman's Rank Correlation Coefficient)~\citep{talebi2018nima}}: Measures monotonic rank correlation. Unlike LCC, SRCC operates on rank variables rather than raw scores, offering greater robustness to outliers and non-linear monotonic relationships.
\end{itemize}

\subsubsection{Implementation Details}
Unless otherwise specified, we adopt the following protocol.
\paragraph{Architecture.} The proposed framework employs an encoder-decoder structure. The encoder is task-specific (detailed in respective sections), while the decoder is consistently instantiated as a two-layer Transformer with 8-head attention for the conditional diffusion process. A dropout rate of 0.1 is applied to mitigate overfitting. The loss weight $\lambda$ in Eq.~(\ref{eq:total_loss}) is set to 10.
\paragraph{Optimization.} We minimize the objective using the Adam optimizer~\citep{kingma2014adam} ($\beta_1=0.9, \beta_2=0.999$) with a learning rate of $5 \times 10^{-4}$.
\paragraph{Training Regimen.} For vision tasks, models are trained for 100 epochs with a batch size of 128. For structured data tasks (WTP, LTV), we train for 50 epochs with a batch size of 1024. All experiments are conducted on a single NVIDIA RTX 4090 GPU.

\subsection{Image Aesthetics Assessment (IAA)}
\label{app:IAA}
\subsubsection{Experimental Setup}
\paragraph{Datasets.}
We evaluate DiffoR on four standard IAA benchmarks: TAD66K~\citep{he2022rethinking}, AVA~\citep{murray2012ava}, ICAA17K~\citep{he2023thinking}, and SPAQ~\citep{fang2020perceptual}. Following standard practice, each dataset is randomly partitioned into 80\% training, 10\% validation, and 10\% testing sets.
\paragraph{Preprocessing.}
Given the narrow range of aesthetic scores (typically 0–10), we scale the target values by a factor of 100 during training to facilitate robust vocabulary construction and ordinal sequencing. For evaluation, all predictions are rescaled to the original range to ensure metric consistency.
\paragraph{Baselines.} We select baselines based on two criteria: (1) classical architectures with open-source implementations; and (2) state-of-the-art (SOTA) methods in specialized domains (e.g., personalized IAA). To guarantee a strictly fair comparison, all baselines are trained using their official hyperparameters and subjected to identical data preprocessing and evaluation protocols, consistent with~\citep{he2023thinking}.

\subsubsection{Performance.}
Due to space constraints in Sec.~\ref{sec:iaa} of the main paper, comprehensive baseline comparisons for the Image Aesthetics Assessment task are presented here. Tab.~\ref{tab:AA1} and Tab.~\ref{tab:AA2} details the performance of all compared methods on the TAD66K, AVA, ICAA17K, and SPAQ datasets.

\begin{table*}[t]
  \centering
  \caption{The results of Image Aesthetics Assessment task on TAD66K and AVA datasets}
  \label{tab:AA1}
  \resizebox{0.66\textwidth}{!}{
  \renewcommand{\arraystretch}{0.85}
  \begin{tabular}{l*{4}{c}*{4}{c}}
    \toprule
    \multirow{2}{*}{Method}
      & \multicolumn{4}{c}{TAD66K}
      & \multicolumn{4}{c}{AVA} \\
    \cmidrule(lr){2-5} \cmidrule(lr){6-9}
       & MAE $\downarrow$ & XAUC $\uparrow$ & LCC $\uparrow$ & SRCC $\uparrow$
       & MAE $\downarrow$ & XAUC $\uparrow$ & LCC $\uparrow$ & SRCC $\uparrow$ \\
    \midrule
    RAPID~\citep{lu2014rapid}
       & 1.766 & 0.510 & 0.332 & 0.314
       & 0.978 & 0.513 & 0.336 & 0.327 \\
    AADB~\citep{kong2016photo}
       & 1.463 & 0.523 & 0.400 & 0.379
       & 0.784 & 0.534 & 0.431 & 0.408 \\
    PAM~\citep{ren2017personalized}
       & 1.314 & 0.534 & 0.440 & 0.422
       & 0.614 & 0.619 & 0.531 & 0.521 \\
    NIMA~\citep{talebi2018nima}
       & 1.422 & 0.511 & 0.405 & 0.390
       & 0.715 & 0.532 & 0.472 & 0.447 \\
    ALamp~\citep{ma2017lamp}
       & 1.349 & 0.523 & 0.422 & 0.411
       & 0.657 & 0.579 & 0.498 & 0.487 \\
    $\text{MP}_{ada}$~\citep{sheng2018attention}
       & 1.191 & 0.589 & 0.408 & 0.389
       & 0.602 & 0.632 & 0.543 & 0.531 \\
    MLSP~\citep{hosu2019effective}
       & 1.132 & 0.620 & 0.432 & 0.409
       & 0.579 & 0.657 & 0.563 & 0.553 \\
    BIAA~\citep{zhu2020personalized}
       & 1.329 & 0.538 & 0.431 & 0.348
       & 0.672 & 0.566 & 0.496 & 0.476 \\
    UIAA~\citep{zeng2019unified}
       & 1.281 & 0.548 & 0.441 & 0.361
       & 0.608 & 0.626 & 0.535 & 0.525 \\
    POE~\citep{li2021learning}
       & 1.185 & 0.588 & 0.420 & 0.377
       & 0.633 & 0.608 & 0.524 & 0.506 \\
    HGCN~\citep{she2021hierarchical}
       & 1.141 & 0.615 & 0.419 & 0.406
       & 0.658 & 0.578 & 0.511 & 0.486 \\
    TANet~\citep{he2022rethinking}
       & 1.081 & 0.649 & 0.452 & 0.428
       & 0.577 & 0.659 & 0.568 & 0.554 \\
    MaxViT~\citep{tu2022maxvit}
       & 1.054 & 0.659 & 0.472 & 0.441 
       & 0.559 & 0.679 & 0.594 & 0.571 \\
    Delegate~\citep{he2023thinking}
       & 1.041 & 0.661 & 0.477 & 0.451 
       & 0.541 & 0.688 & 0.642 & 0.634 \\
    AesMamba~\citep{gao2024aesmamba}
       & 1.035 & 0.666 & 0.482 & 0.468 
       & 0.522 & 0.697 & 0.663 & 0.656 \\
    GoR~\citep{ma2026gor}
       & 0.996 & 0.677 & 0.541 & 0.513 
       & 0.395 & 0.751 & 0.726 & 0.701 \\
    \textbf{DiffoR (Ours)} 
       & \textbf{0.914}& \textbf{0.694} & \textbf{0.553} & \textbf{0.526}
       & \textbf{0.388}& \textbf{0.766} & \textbf{0.734} & \textbf{0.725} \\
    \bottomrule
  \end{tabular}
  }
\end{table*}

\begin{table*}[t]
  \centering
  \caption{The results of the Image Aesthetics Assessment task on ICAA17K and SPAQ datasets.}
  \label{tab:AA2}
  \resizebox{0.66\textwidth}{!}{
  \renewcommand{\arraystretch}{0.85}
  \begin{tabular}{l*{4}{c}*{4}{c}}
    \toprule
    \multirow{2}{*}{Method}
      & \multicolumn{4}{c}{ICAA17K}
      & \multicolumn{4}{c}{SPAQ} \\
    \cmidrule(lr){2-5} \cmidrule(lr){6-9}
       & MAE $\downarrow$ & XAUC $\uparrow$ & LCC $\uparrow$ & SRCC $\uparrow$
       & MAE $\downarrow$ & XAUC $\uparrow$ & LCC $\uparrow$ & SRCC $\uparrow$ \\
    \midrule
    RAPID~\citep{lu2014rapid}
       & 0.7415 & 0.6416 & 0.5164 & 0.5083
       & 1.0890 & 0.6997 & 0.6565 & 0.6128 \\
    AADB~\citep{kong2016photo}
       & 0.7142 & 0.6661 & 0.5311 & 0.5195
       & 1.083 & 0.7036 & 0.6646 & 0.6162 \\
    PAM~\citep{ren2017personalized}
       & 0.7070 & 0.6729 & 0.5385 & 0.5247
       & 1.0726 & 0.7104 & 0.6691 & 0.6222 \\
    ALamp~\citep{ma2017lamp}
       & 0.6948 & 0.6847 & 0.5478 & 0.5339
       & 1.0511 & 0.7250 & 0.6835 & 0.6349 \\
    NIMA~\citep{talebi2018nima}
       & 0.6957 & 0.6839 & 0.5458 & 0.5333
       & 1.0756 & 0.7084 & 0.6709 & 0.6204 \\
    $\text{MP}_{ada}$~\citep{sheng2018attention}
       & 0.6948 & 0.6848 & 0.5485 & 0.5340
       & 1.0525 & 0.7240 & 0.6808 & 0.6341 \\
    MLSP~\citep{hosu2019effective}
       & 0.6814 & 0.6983 & 0.5606 & 0.5445
       & 1.0428 & 0.7306 & 0.6952 & 0.6402 \\
    MT-A~\citep{fang2020perceptual}
       & 0.6855 & 0.6940 & 0.5558 & 0.5412
       & 1.0455 & 0.7289 & 0.6862 & 0.6384 \\
    BIAA~\citep{zhu2020personalized}
       & 0.6864 & 0.6932 & 0.5552 & 0.5405
       & 1.0497 & 0.7259 & 0.6826 & 0.6358 \\
    UIAA~\citep{zeng2019unified}
       & 0.6889 & 0.6907 & 0.5559 & 0.5386
       & 1.0469 & 0.7278 & 0.6862 & 0.6376 \\
    POE~\citep{li2021learning}
       & 0.6808 & 0.6966 & 0.5583 & 0.5432
       & 1.0456 & 0.7307 & 0.6877 & 0.6368\\
    MUSIQ~\citep{ke2021musiq}
       & 0.6740 & 0.7059 & 0.5632 & 0.5504
       & 1.0427 & 0.7308 & 0.6925 & 0.6401 \\
    HGCN~\citep{she2021hierarchical}
       & 0.6813 & 0.6983 & 0.5566 & 0.5445 
       & 1.040 & 0.7328 & 0.6934 & 0.6417 \\
    TANet~\citep{he2022rethinking}
       & 0.6789 & 0.7008 & 0.5599 & 0.5465 
       & 1.0469 & 0.7279 & 0.6844 & 0.6375 \\
    MaxViT~\citep{tu2022maxvit}
       & 0.6582 & 0.7227 & 0.5853 & 0.5636
       & 1.042 & 0.7308 & 0.6925 & 0.6401 \\
    Delegate~\citep{he2023thinking}
       & 0.6345 & 0.7498 & 0.6034 & 0.5847 
       & 1.019 & 0.7473 & 0.7114 & 0.6545 \\
    AesMamba~\citep{gao2024aesmamba}
       & 0.6129 & 0.7663 & 0.6137 & 0.6294 
       & 0.9875 & 0.7522 & 0.7261 & 0.6895 \\
    GoR~\citep{ma2026gor}
       & 0.5842 & 0.7913 & 0.6823 & 0.6789
       & 0.8722 & 0.7648 & 0.7434 & 0.7233 \\
    \textbf{DiffoR (Ours)} 
       & \textbf{0.552}& \textbf{0.841} & \textbf{0.759} & \textbf{0.765}
       & \textbf{0.768}& \textbf{0.805} & \textbf{0.799} & \textbf{0.785} \\
    \bottomrule
  \end{tabular}
  }
\end{table*}


\subsection{Watch Time Prediction (WTP)}
\label{app:wtp}
\subsubsection{Experimental Setup}
\label{app:wtp_setting}
\paragraph{Dataset.}
We evaluate our method on two large-scale real-world datasets derived from Kuaishou video logs. KuaiRand~\citep{gao2022kuairand} comprises 26,988 users and 6,598 items, generating a total of 1,266,560 interaction impressions. KuaiRec~\citep{gao2022kuairec}, a significantly larger collection, consists of 7,176 users and 10,728 items, accumulating an impressive 12,530,806 impressions.
\paragraph{Architecture}
Unlike traditional sequential recommendation tasks, Watch Time Prediction (WTP) does not inherently rely on historical action sequences. Accordingly, we employ a simple two-layer Multi-Layer Perceptron (MLP) as the encoder, maintaining consistency with the configuration of baseline methods to ensure fair comparison.
\paragraph{Baselines.}
Details about these compared methods are as follows:
\begin{enumerate}[leftmargin=20pt]
    \item \textbf{VR (Value Regression):} This method employs direct regression fitting to predict the absolute watch time prediction, optimized via Mean Squared Error (MSE).
    \item \textbf{D2Q~\cite{d2q} :} This method segments data by video duration, predicting watch time quantiles within groups via regression, then mapping to final watch time.
    \item \textbf{D2Co~\cite{d2co}:} introduces a sensitivity-controlled correction mechanism to mitigate both duration bias and noisy watching, aiming to recover true user interest via a unified causal framework.
    \item \textbf{CWM~\cite{cwm}:} It models counterfactual watch time (CWT) by estimating user interest via a cost-based transform function, optimizing a counterfactual likelihood for prediction.
    \item \textbf{TPM~\cite{tpm}:} It utilizes a tree structure to model multi-granular time interval relationships, predicting watch time as a weighted sum of probabilities along the tree path.
    \item \textbf{PTPM~\cite{PTPM}:} It replaces the pre-defined binary tree with a personalized, learnable structure in TPM with adaptive ordinal discretization and robust bias correction.
    \item \textbf{CREAD~\cite{cread}:} This method constructs dynamic, error-adaptive time intervals. Within each, a classifier determines threshold exceedance, deriving the final prediction from a weighted sum of probabilistic estimates.
    \item \textbf{SWaT~\cite{swat}:} User-centric statistical framework modeling watch time with behavioral assumptions. It employs bucketization for non-stationary viewing probabilities, with prediction via a weighted sum of probabilistic estimates.
    \item \textbf{EGMN~\cite{egmn}:} It parameterizes an Exponential-Gaussian Mixture to simultaneously characterize the coarse-grained skewness of quick skips and the fine-grained diversity of user interactions.
    \item \textbf{GoR~\cite{ma2026gor}:} It employs an autoregressive language modeling approach to generate watch time tokens, where the final predicted duration is the sum of these generated tokens.
\end{enumerate}
For a fair comparison, all compared methods are implemented with their reported optimal hyperparameters and configured to maintain approximate model parameter sizes.

\subsection{Life Time Value Prediction (LTV)}
\label{app:ltv}
\subsubsection{Experimental Setup}
We evaluate DiffoR on the Criteo-SSC\footnote{https://ailab.criteo.com/criteo-sponsored-search-conversion-log-dataset/} and Kaggle\footnote{https://www.kaggle.com/c/acquire-valued-shoppers-challenge} datasets.
For both datasets, a random split of 7:1:2 is used for training, validation, and testing, respectively.
Criteo-SSC is a large-scale public dataset derived from Criteo Predictive Search (CPS) logs. Each instance represents a user's click behavior, with the task being to predict conversion and associated 30-day revenue. The product price feature was excluded from the inputs.
The Kaggle Dataset contains transaction records. Following~\citep{weng2024optdist}, the task involves predicting a user's total purchase value from a specific company in the year following their initial purchase. Our experiments focus on initial purchases within 2012-03-01 and 2012-07-01, using data from the three companies with the highest transaction volume.

\subsubsection{Baselines}
We evaluate our method with several existing state-of-the-art LTV methods~\citep{drachen2018or,ma2018entire,wang2019deep,li2022billion,liu2024mdan,weng2024optdist}. Here, we provide more
detailed information about these compared methods as follows:
\begin{enumerate}[itemsep=0.5ex, parsep=0pt, topsep=0pt, leftmargin=*, align=left]
\item \textbf{Two-stage}~\citep{drachen2018or} decomposes the CLTV prediction into two tasks: the first task is a classification task predicting whether a user will churn or not, and the second task is a regression task predicting the revenue that the user brings.
\item \textbf{MTL-MSE}~\citep{ma2018entire} estimates conversion rate and CLTV with MSE loss according to the multi-task learning paradigm. 
\item \textbf{ZILN}~\citep{wang2019deep} assumes that the long-tailed CLTV distribution follows a zero-inflated log-normal distribution and uses a DNN to estimate the mean $\mu$, standard deviation $\sigma$  and conversion rate $p$ for the samples.
\item \textbf{MDME}~\citep{li2022billion} divides the training samples by CLTV into multiple sub-distributions and buckets, and constructs corresponding classification problems to predict the bucket a sample belongs to. In the next stage, the bias within the bucket is estimated so that the samples obtain a fine-grained CLTV value.
\item \textbf{MDAN}~\citep{liu2024mdan} predicts predefined LTV bucket labels using a multi-classification network and leverages a multi-channel learning network to derive embeddings for each bucket. The final sample representation is obtained by fusing these embeddings with the classification network’s output through a weighted sum, which is then utilized for CLTV prediction.
\item \textbf{OptDist}~\citep{weng2024optdist} employs an adaptive mechanism to model and select optimal sub-distributions for individual samples, consisting of a distribution learning module (DLM) that trains multiple sub-distribution networks, and a distribution selection module (DSM) that dynamically chooses the appropriate sub-distribution for each customer.
\item \textbf{HiLTV}~\citep{xu2025hiltv} is a hierarchical framework for game LTV prediction that models multi-modal recharge behaviors with a Zero-Inflated Mixture-of-Logistic loss and introduces a calibration module for robust new-user prediction.
\end{enumerate}
For this task, we employ the same encoder architecture for DiffoR in Appendix~\ref{app:wtp_setting}. 

\end{document}